%% file: main.tex
\documentclass{article}
\usepackage{iclr2027_conference,times}
\usepackage[T1]{fontenc}
\usepackage[utf8]{inputenc}
\usepackage{amsmath,amssymb,amsthm,mathtools}
\usepackage{booktabs,array,tabularx}
\usepackage{graphicx,xcolor,float}
\usepackage{algorithm,algpseudocode}
\usepackage{hyperref,url,microtype}

\hypersetup{
  hidelinks,
  pdftitle={Evidence Acquisition and Stopping via Lattice Distillation},
  pdfauthor={Yixuan Zhang, Yang Song, Jiayi Zhang, Jingqi Sun, Hao Wang, Samir Bhatt, Hengguan Huang}
}

\input{macros}
\input{title}

\title{\centering
  Evidence Acquisition and Stopping\\
  via Lattice Distillation
}
\author{%
  \parbox{\dimexpr\textwidth-2\tabcolsep\relax}{%
    \centering
    \normalfont\normalsize
    \textbf{Yixuan Zhang}\textsuperscript{1}\quad
    \textbf{Yang Song}\textsuperscript{1}\quad
    \textbf{Jiayi Zhang}\textsuperscript{2}\quad
    \textbf{Jingqi Sun}\textsuperscript{1}
    \\[0.45em]
    \textbf{Hao Wang}\textsuperscript{3}\quad
    \textbf{Samir Bhatt}\textsuperscript{1,4,*}\quad
    \textbf{Hengguan Huang}\textsuperscript{1,4}\thanks{%
      \raggedright
      Correspondence to:
      Hengguan Huang (\texttt{huang.hengguan@u.nus.edu})
      and Samir Bhatt (\texttt{s.bhatt@imperial.ac.uk}).
      Hengguan Huang is the lead corresponding author.
    }
    \\[0.9em]
    \small
    \textsuperscript{1}Section of Health Data Science \& AI,
    Department of Public Health,\\
    University of Copenhagen, Copenhagen, Denmark
    \\[0.4em]
    \textsuperscript{2}Shanghai Jiao Tong University, Shanghai, China
    \\[0.4em]
    \textsuperscript{3}University of Illinois at Urbana-Champaign,
    Urbana, IL, USA
    \\[0.4em]
    \textsuperscript{4}MRC Centre for Global Infectious Disease Analysis,\\
    Department of Infectious Disease Epidemiology,
    School of Public Health,\\
    Faculty of Medicine, Imperial College London,
    London, United Kingdom
  }%
}
\iclrfinalcopy

\begin{document}

\maketitle
\fancyhead{}
\renewcommand{\headrulewidth}{0pt}

\input{sections/00_abstract_intro}
\input{sections/01_related_problem}
\input{sections/02_masked_representation}
\input{sections/03_lattice_distillation}
\input{sections/04_theory}
\input{sections/06_experiments}
\input{sections/07_design_ablations}
\input{sections/08_replanning_transfer}
\input{sections/09_conclusion}

\label{main:lastpage}
\input{sections/statements}

\bibliographystyle{iclr2027_conference}
\bibliography{references}

\clearpage
\appendix
\input{sections/appendix}

\end{document}

%% file: macros.tex
\newtheorem{theorem}{Theorem}
\newtheorem{proposition}[theorem]{Proposition}
\newtheorem{corollary}[theorem]{Corollary}
\newtheorem{lemma}[theorem]{Lemma}
\theoremstyle{remark}

\newcommand{\E}{\mathbb{E}}
\newcommand{\M}{\mathcal{M}}
\newcommand{\A}{\mathcal{A}}
\newcommand{\ind}{\mathbf{1}}
\newcommand{\pos}[1]{[#1]_+}
\newcommand{\KL}{D_{\mathrm{KL}}}
\newcommand{\SLD}{\textsc{SLD}}

\newcommand{\stopa}{\mathsf{stop}}

\newcommand{\argmax}{\operatorname*{arg\,max}}

%% file: title.tex
\title{Evidence Acquisition And Stopping Via Lattice Distillation}

%% file: sections/00_abstract_intro.tex
\begin{abstract}
Sequential evidence acquisition requires deciding what to observe and when to stop from partial information. We introduce a complete-record view of policy supervision: each fully observed training instance defines a lattice of evidence subsets, with a stop/query decision at every node. For a modest number of evidence groups, this structure can be enumerated even when individual groups contain rich representations. Our method, \emph{Subset Lattice Distillation} (\SLD), scores subsets with a shared masked predictor, constructs multi-step decision targets by cost-sensitive backward planning, and distills them into an observed-history controller. Uniformly resampled masks train the student across the subset space, while combinatorial planning remains offline. Separating subset scores from acquisition prices allows the predictor to be reused when the policy is replanned for new costs. Our analysis identifies action-margin conditions under which privileged supervision preserves optimal online decisions. On retrospective influenza monitoring, \SLD{} achieves higher phase F1 at lower acquisition cost than both the closest-cost adaptive baselines and the fixed full-evidence model. A geographic and seasonal holdout, price perturbations, and ablation studies provide further evaluation.
\end{abstract}

\section{Introduction}
\label{sec:introduction}
Useful evidence is not always available at the start of a prediction task. Additional inputs may improve an assessment, but obtaining them incurs a cost. When inputs are acquired as groups, such as a clinical stream or a contextual report, the relevant choice is which group to request for the current instance and when to stop collecting evidence. These decisions are coupled: acquiring a group changes both the prediction and the value of subsequent acquisitions. A cost-sensitive policy compares stopping with the terminal loss and cumulative cost of continuing.

Retrospective influenza monitoring provides a concrete setting. Wastewater measurements (WW) are available initially; emergency-department visits (ED), hospital admissions, laboratory positivity, and source-context summaries are optional evidence groups. Each state--week record is labelled for epidemic phase and alert status. A fixed WW+ED workflow uses wastewater and emergency-department evidence for every episode, whereas an adaptive policy can acquire different groups or stop with a smaller subset. Each query reveals aligned evidence from the stored record and incurs its assigned acquisition cost.

Active feature acquisition learns such policies through reinforcement learning \citep{shim2018}, models of partially observed data \citep{ma2019,li2021}, and greedy dynamic feature selection (GDFS) \citep{covert2023}. Nongreedy oracle supervision, including the acquisition conditioned oracle (ACO), accounts for jointly informative subsets \citep{valancius2024}. We study the construction of supervision from complete grouped records. An outcome label evaluates a prediction, but does not directly specify how to act under each partial observation. The learning problem is to turn those records into cost-sensitive targets for alternative acquisition paths, while keeping the deployed controller restricted to the evidence it has acquired.

Our key observation is that \emph{rich evidence representations can coexist with a small acquisition space}. With a modest number of groups, observed subsets form an enumerable lattice: each node is an evidence mask, and each edge adds one group. A complete record supplies the values needed to evaluate every node. Backward planning reuses successor values to construct a table of decision targets. We treat each complete record as a structured collection of acquisition and stopping problems, accessed through uniformly resampled masks during policy learning.

We implement this perspective with \emph{Subset Lattice Distillation} (\SLD). A shared masked predictor first scores the evidence subsets of each record. A cost-sensitive teacher then compares stopping at each subset with acquiring a group and continuing through its successors. Finally, a student learns from uniformly resampled masks using the corresponding evidence representation, cost parameters, and stop/query targets. This construction places multi-step planning in the training stage; deployment requires only encoding the observed evidence and applying the controller. Subset scoring is separate from acquisition prices, so the predictor can be reused when new prices call for replanning and redistillation.

The lattice teacher evaluates realized continuations, while the student learns from partial observations. We characterize the conditional preferences induced by soft-target distillation and derive decision-margin conditions for their transfer to online control. The analysis separates changes in action ordering from changes in confidence, and quantifies robustness to student error and approximate planning.

We evaluate \SLD{} alongside fixed evidence workflows and adaptive acquisition methods. It improves phase prediction over ACO and GDFS at lower acquisition cost on the retrospective influenza benchmark. We assess geographic and seasonal transfer using a held-out test set. Price-profile experiments examine replanning with a fixed predictor, while component studies assess sensitivity to representation and policy learning.

\paragraph{Contributions.}
\begin{itemize}
\item \textbf{A complete-record supervision perspective.} We organize each training record into a reusable table of subset scores and multi-step acquisition decisions, separating evidence evaluation from cost-sensitive planning.
\item \textbf{A lattice-distillation algorithm.} We introduce \SLD, combining shared subset evaluation, cost-sensitive backward planning, and observed-history distillation, with a soft-target analysis and sufficient conditions for preserving online decisions.
\item \textbf{Empirical evaluation.} \SLD{} improves phase prediction over the closest-cost adaptive baselines while incurring lower acquisition cost; holdout, price-profile, and component studies examine transfer and acquisition behavior.
\end{itemize}
\label{intro:end}

%% file: sections/01_related_problem.tex
\section{Related Work}
Cost-aware acquisition methods balance predictive benefit against the expense of obtaining additional evidence, learning when to query and when to stop \citep{trapeznikov2013,shim2018}. Their central design choice is how to assess unobserved information. EDDI estimates information gain with a partial generative model, while generative surrogate methods learn feature dependencies to support acquisition \citep{ma2019,li2021}. These approaches model missing evidence; \SLD{} instead scores observed realizations in complete training records. GDFS avoids explicit density estimation through a greedy conditional-information objective \citep{covert2023}, but its one-step criterion can miss benefits that emerge only after further acquisition. Stochastic encodings reason over possible realizations in a learned latent space \citep{norcliffe2025}; in \SLD{}, a shared representation supplies subset predictions and a separate teacher evaluates multi-step continuations. Our complete-record perspective exploits a modest number of evidence groups to turn an enumerable table of subset scores into stop/query supervision, without fitting a generative model of missing groups or restricting the teacher to immediate predictive gain.

A complementary question is how training-time information becomes supervision for a deployable policy. Asymmetric imitation shows that a fully informed teacher can induce unsuitable actions under partial observation \citep{warrington2021}; cost-sensitive imitation relates action errors to decision cost \citep{ross2014}. Direct value learning instead regresses continuation costs, as in BRiG-AFA \citep{feng2026}. Within oracle-based acquisition, ACO evaluates nongreedy candidate subsets conditional on the current observations and learns its policy by behavioral cloning of oracle rollouts\citep{valancius2024}. \SLD{} instead constructs a realized target lattice for each complete record and trains on uniformly resampled masks. The lattice organizes multi-step preferences independently of a particular rollout, while sampling provides stochastic updates from the full subset space. 
% This complete-record perspective supports reusable subset evaluation and cost-sensitive replanning, with a soft-target analysis linking learned preferences to online decision cost.

\section{Problem Formulation}
\label{sec:problem}
We study \emph{cost-sensitive grouped evidence acquisition}: predicting a target while deciding which additional evidence to obtain and when to stop. Each episode provides mandatory evidence $X^{(0)}$, optional groups $\{X^{(m)}:m\in\M\}$, and a target $Y$, where $\M$ indexes the $K=|\M|$ optional groups. After acquiring $S\subseteq\M$, the observed history $H_S$ contains the mandatory input, acquired group values, their binary mask, and available metadata. A policy $\pi$ uses $H_S$ to query an unobserved group $m\in\M\setminus S$ at cost $c_m\geq0$, or to stop and invoke a terminal prediction rule. Each group is acquired at most once. The objective is to minimize expected terminal loss plus accumulated acquisition cost. Training uses complete labeled records; deployment reveals only queried values and never the target. We model each query as revealing stored evidence within an episode, not as waiting for new measurements over calendar time. Optional deferral is a terminal-output choice after acquisition stops.

%% file: sections/02_masked_representation.tex
\section{Subset Lattice Distillation}
\label{sec:method}
\paragraph{Overview.}
A complete record can supervise decisions at multiple evidence subsets, so \SLD{} moves the combinatorial comparison of acquisition paths into training. The method has three components: \emph{shared subset evaluation} turns partial observations into prediction scores; a \emph{cost-sensitive lattice teacher} converts those scores into multi-step stop/query targets; and \emph{observed-history distillation} learns a deployable controller from the targets. Figure~\ref{fig:architecture} shows this flow and the separation between complete training records and partial observations at deployment. Subset evaluation is performed before prices enter the planner, allowing the predictor to be reused under new acquisition costs.

\begin{figure}[!htbp]
\centering
\includegraphics[width=\linewidth]{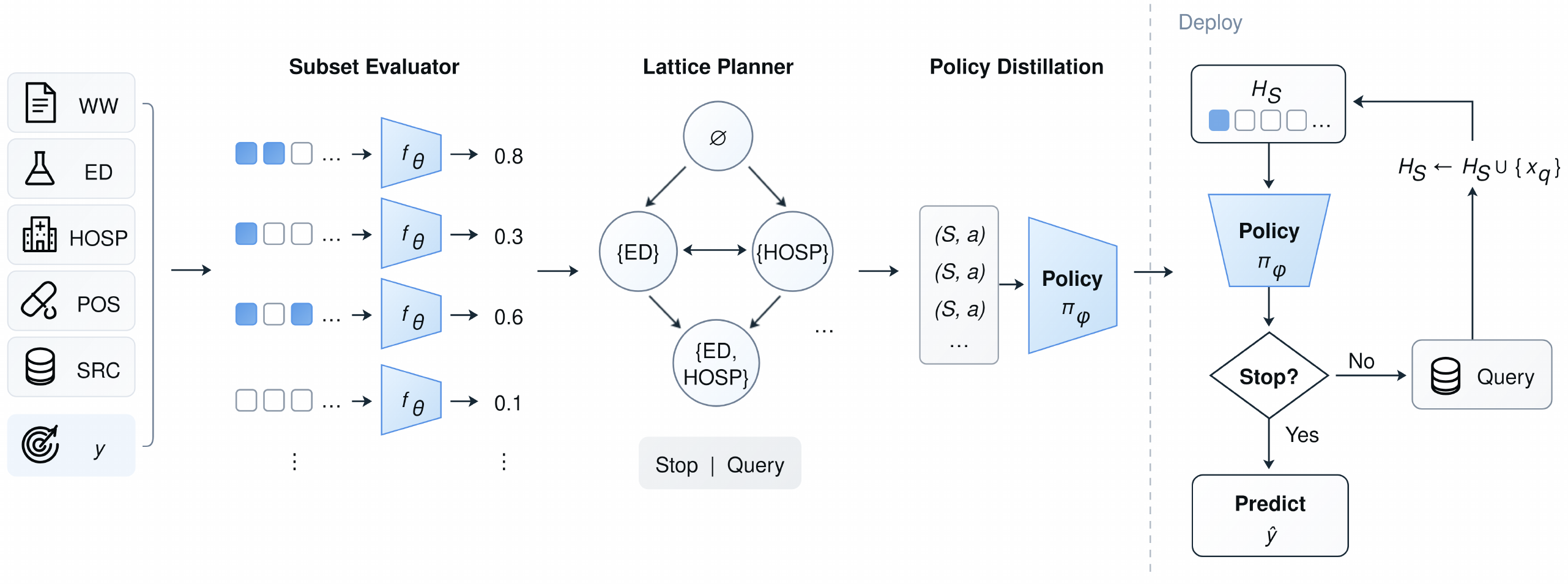}
\caption{\textbf{From complete records to acquisition policies.} Shared subset evaluation produces one score table per training record. A cost-sensitive lattice teacher uses backward planning to assign stop/query targets, and uniformly resampled masks train an observed-history controller. Online, a query reveals one group and updates the history; a stop invokes the terminal rule. Labels and unqueried values enter offline target construction, not the deployed controller's inputs.}
\label{fig:architecture}
\end{figure}

\subsection{Shared subset evaluation}
\label{sec:subset}
Comparing acquisition paths requires predictions from different evidence subsets. A single masked predictor shares statistical strength across these subsets and avoids fitting a separate model for each mask. The input retains observed blocks, zero-fills unobserved numeric and text blocks, and includes a binary observation mask so that missing blocks remain distinguishable from measured zeros. The encoder and decoder are
\begin{equation}
 q_\theta(z\mid H_S)=\mathcal N\!\left(\mu_S,\operatorname{diag}(\sigma_S^2)\right),\quad
 f_\psi(H_S)=\operatorname{decode}_\psi\!\left([\mu_S;\log\sigma_S^2]\right).
\label{eq:encoder}
\end{equation}
Here $z\in\mathbb R^d$ is a latent vector of dimension $d$; $\mu_S\in\mathbb R^d$ and $\sigma_S^2\in\mathbb R_{>0}^d$ are its mean and variance. The operator $\operatorname{diag}$ forms a diagonal covariance matrix, brackets with a semicolon denote concatenation, and the logarithm is elementwise. The parameters $\theta$ belong to the encoder and $\psi$ to the decoder. We write $f(H)$ for the prediction produced by this model at history $H$. Its Gaussian moments summarize observed evidence; the model does not generate missing groups.

To train one predictor across masks, let $t$ index a complete training record and $H_{t,S}$ its history under subset $S$. In the influenza task, $Y_t=(y_t,b_t)$ consists of a four-class phase label $y_t$ and a binary alert label $b_t$. The decoder defines their predictive distribution $p_\psi(y_t,b_t\mid H_{t,S})$ and is trained jointly with the encoder using
\begin{equation}
 \mathcal L_{\rm state}=\E_{t,S}\big[-\log p_\psi(y_t,b_t\mid H_{t,S})
 +\beta\KL\!\left(q_\theta(\cdot\mid H_{t,S})\Vert q_\theta(\cdot\mid H_{t,\varnothing})\right)\big].
\label{eq:state}
\end{equation}
The expectation averages training records and sampled subsets; $\beta\geq0$ weights the Kullback--Leibler divergence $\KL$. The empty subset $\varnothing$ retains only mandatory evidence, providing an instance-specific reference for the regularizer. Task likelihood learns predictive content, while the regularizer discourages unconstrained changes between subset representations. Observed text summaries are encoded with a frozen text model; only their representations, not a text-generated action, enter the predictor.

The planner needs a scalar cost for stopping at each subset, rather than a prediction distribution alone. After freezing $\theta$ and $\psi$, a fixed scoring procedure produces a realized terminal score $\rho_t(S)$ from each complete record. With compulsory prediction, this can be a prediction-loss score. With an optional reject decision \citep{geifman2019}, a generic construction uses a calibrated realized prediction score $\widetilde r_t(S)$ and a deferral cost $\tau>0$:
\begin{equation}
 \rho_t(S)=\min\{\tau,\widetilde r_t(S)\},\qquad
 d_t(S)=\ind\{\widetilde r_t(S)<\tau\}.
 \label{eq:canonical}
\end{equation}
The indicator $\ind$ equals one when its condition holds, so $d_t(S)$ is a training target for predicting rather than deferring after a stop. These are realized supervision targets, not conditional expected losses. In our experiments, we use the task-specific terminal scoring and output procedure described in (Appendix~\ref{app:implementation}); Equation~\eqref{eq:canonical} specifies the generic optional-deferral interface. The resulting table $\{\rho_t(S):S\subseteq\M\}$ is the input to the lattice teacher.

%% file: sections/03_lattice_distillation.tex
\subsection{Cost-sensitive lattice teacher}
\label{sec:teacher}
A subset score alone does not indicate whether another group is worth acquiring: a useful group may enable further acquisitions before prediction improves. The teacher therefore compares stopping with complete continuations under one loss--cost objective. For a policy $\pi$ including its terminal rule, let $T$ be its stopping time, $S_T$ its final acquired subset, and $L_T$ its terminal loss. The objective is
\begin{equation}
 J_\lambda(\pi)=\E_\pi\!\left[L_T+\lambda\sum_{m\in S_T}c_m\right],
 \label{eq:objective}
\end{equation}
where $\lambda\geq0$ converts acquisition costs to the terminal-loss scale. The expectation includes the record distribution and any policy randomization. A prediction incurs $L_T=\ell(f(H_{S_T}),Y)$ for a nonnegative integrable task loss $\ell$; optional deferral incurs $L_T=\tau$. Thus stopping is compared with the loss and total cost of continuing, not with the next query's predictive gain alone.

For each record, the $2^K$ subsets form an acyclic lattice whose edges add one group. A complete record supplies a realized successor for every edge. Starting from the full subset, where no query remains, the teacher computes
\begin{equation}
 \widetilde V_{\lambda,t}(S)=\min\!\left\{\rho_t(S),\min_{m\in\M\setminus S}
 \left[\lambda c_m+\widetilde V_{\lambda,t}(S\cup\{m\})\right]\right\},
 \label{eq:teacher}
\end{equation}
with boundary $\widetilde V_{\lambda,t}(\M)=\rho_t(\M)$. Here $\widetilde V_{\lambda,t}(S)$ is the minimum remaining cost for record $t$ at mask $S$. The first term stops immediately; each query term pays for $m$ and uses the already-computed successor value. Processing subsets in decreasing size reuses these successor values across paths. The recurrence is exact for the realized record and its terminal scores. Its outputs provide cost-sensitive preferences; Section~\ref{sec:theory} analyzes how learning them from partial observations relates to online decision cost.

To obtain supervision for the controller, write $\A_S^{\rm acq}=\{\stopa\}\cup(\M\setminus S)$ for the valid actions. Define $\widetilde Q_{\lambda,t}(\stopa,S)=\rho_t(S)$ and $\widetilde Q_{\lambda,t}(m,S)=\lambda c_m+\widetilde V_{\lambda,t}(S\cup\{m\})$ for a query. Near-tied action costs can make hard targets unstable, so we use the smoothed distribution
\begin{equation}
 p_T(a\mid t,S,\lambda)=
 \frac{\exp[-\widetilde Q_{\lambda,t}(a,S)/T_Q]}
 {\sum_{a'\in\A_S^{\rm acq}}\exp[-\widetilde Q_{\lambda,t}(a',S)/T_Q]}.
 \label{eq:soft}
\end{equation}
Here $a$ is an available action, $a'$ ranges over alternatives, and $T_Q>0$ is the target temperature, distinct from the stopping time $T$. The subscript $T$ on $p_T$ denotes the teacher. Smoothing gives similar supervision to similarly priced actions; the planning recurrence itself retains the hard minimum.

The target table spans every mask of each complete record. Student training accesses this table through repeated uniform mask sampling. Because subset scores $\rho_t(S)$ are separate from prices $c=(c_m)_{m\in\M}$, changing prices requires replanning and redistillation while reusing the predictor and score table when the terminal procedure is fixed.

\subsection{Observed-history distillation and deployment}
\label{sec:distillation}
The teacher's full record cannot be exposed to the deployed controller. Distillation transfers its action targets while restricting each student input to the corresponding partial history. Let $s(H_S)$ concatenate the encoder moments, evidence mask, available metadata, and cost features. A policy $\pi_\eta(a\mid s(H_S),\lambda)$, with parameters $\eta$, predicts a distribution over $\A_S^{\rm acq}$. When deferral is enabled, a terminal head $h_\omega(s(H_S))\in(0,1)$, with parameters $\omega$, predicts the answer indicator $d_t(S)$. We train them with
\begin{equation}
\begin{split}
 \mathcal L_{\rm distill}=\E_{t,S,\lambda}\Big[&-\sum_{a\in\A_S^{\rm acq}}
 p_T(a\mid t,S,\lambda)\log\pi_\eta(a\mid s(H_{t,S}),\lambda)\\
 &+\gamma\operatorname{BCE}\!\left(d_t(S),h_\omega(s(H_{t,S}))\right)\Big].
\end{split}
 \label{eq:distill}
\end{equation}
The first term is policy cross-entropy; the second is binary cross-entropy (BCE) for the terminal output, weighted by $\gamma\geq0$. Masks are repeatedly sampled from $w=\operatorname{Unif}(2^\M)$, where $2^\M$ is the set of all subsets and $w_S=2^{-K}$. Each minibatch combines records, fresh masks, and training multipliers. Uniform resampling is an unbiased stochastic estimate of the full-mask average; individual updates need not traverse the lattice. Compulsory prediction omits the terminal head. Labels and complete evidence supply targets only.

At deployment, $S=\varnothing$ and only mandatory evidence is visible. Each query reveals one group and updates $H_S$; the next controller evaluation can select another query or stop. Repeated or unavailable queries are masked, and the full subset terminates. A stop invokes the terminal procedure: the generic deferral head predicts above a development-selected threshold and otherwise defers, while the experiments use the procedure in Appendix~\ref{app:implementation}. Algorithm~\ref{alg:sld} summarizes training and deployment.

\begin{algorithm}[!htbp]
\caption{\SLD: complete-record supervision and observed-history control}
\label{alg:sld}
\begin{algorithmic}[1]
\Require Complete training records $\mathcal D_{\rm tr}$, disjoint development records $\mathcal D_{\rm dev}$, prices $c$, multiplier grid $\Lambda$, terminal scoring procedure.
\State Fit the masked predictor using Equation~\eqref{eq:state}; freeze its parameters.
\State Score all training record--mask pairs: $\rho_t(S)$ and, when used, $d_t(S)$.
\For{$\lambda\in\Lambda$ and each training record $t$}
  \State Solve Equation~\eqref{eq:teacher} in decreasing subset size.
  \State Store the stop/query targets in Equation~\eqref{eq:soft} at every mask.
\EndFor
\For{each student update}
  \State Sample records and multipliers; redraw $S\sim\operatorname{Unif}(2^\M)$.
  \State Update the student and optional terminal head with Equation~\eqref{eq:distill}.
\EndFor
\State Fix operating parameters on $\mathcal D_{\rm dev}$; report realized student cost.
\Statex \textbf{Online:} initialize $S=\varnothing$ with mandatory evidence observed.
\While{the controller selects an available unobserved group $m$}
 \State Pay $c_m$, reveal the group, and update $H_S$ and $S$.
\EndWhile
\State Stop and invoke the terminal rule on $H_S$.
\end{algorithmic}
\end{algorithm}

\paragraph{Computational scope and scaling.}
Let $N=|\mathcal D_{\rm tr}|$ and $|\Lambda|$ be the number of training multipliers. Exact target construction uses $N2^K$ subset evaluations and $O(N|\Lambda|K2^K)$ scalar planning operations. Here $K$ counts acquisition \emph{groups}, not the dimensionality within a group; the present benchmark has four groups and therefore 16 subset states. Student updates evaluate resampled masks, and deployment requires at most $K$ queries and $K+1$ controller evaluations with no lattice search.

Extending the offline planner to a sparse subset DAG is straightforward: the same backward recurrence is applied only to retained nodes and edges. For a sampled supervision mask, retain every feasible first query and use a budgeted search for its continuations. Once the graph is constructed, evaluation and planning scale with its nodes and edges, rather than $2^K$. The student and deployment interface are unchanged. Appendix~\ref{app:sparse} gives the construction cost and a theorem relating pruning error to decision stability.
\label{method:end}

%% file: sections/04_theory.tex
\section{From Lattice Supervision to Online Decisions}
\label{sec:theory}
Fix the fitted functions and costs. Let $W\sim P$ be a fresh complete record, $S\sim\operatorname{Unif}(2^\M)$ an independent mask, $\nu=P\times w$, and $U=s(H_S)$ the student input, including the mask. We analyze the preferences learned by distillation and when their ordering preserves online decisions.

\subsection{What the student learns}
\begin{proposition}[Conditional preference projection]
\label{prop:projection}
The policy cross-entropy over unrestricted observed-input distributions is minimized by
\begin{equation}
 \bar p(a\mid U)=\E_\nu[p_T(a\mid W,S)\mid U].
 \label{eq:projection}
\end{equation}
\end{proposition}
The projection follows asymmetric imitation \citep{warrington2021}. Its KL decomposition appears in Appendix~\ref{app:projection}; we focus here on action ordering.

\subsection{Decision stability under privileged supervision}
Fix the deployed predictor and terminal rule $g$. Let $Q_g^*(H,a)$ be the expected cost of action $a$ followed by optimal observed-history acquisition using the same $g$. Write $\bar q(U)=\E_\nu[\widetilde Q_W(\cdot,S)\mid U]$ and $\operatorname{osc}(x)=\max_a x_a-\min_a x_a$. Define
\begin{equation}
\begin{split}
 b(H)&=\operatorname{osc}\!\left(\bar q(U)-Q_g^*(H,\cdot)\right),\quad
 e(U)=\|\pi_\eta(\cdot\mid U)-\bar p(\cdot\mid U)\|_\infty,\\
 v^2(U)&=\E_\nu\!\left[\operatorname{osc}\!\left(\widetilde Q_W-\bar q(U)\right)^2\mid U\right].
\end{split}
\label{eq:decision_quantities}
\end{equation}
Here $b$ measures action-relative mean distortion, including terminal-score mismatch and representation compression; $v^2$ measures fluctuations of cost differences, and $e$ is student error. At a history with $n>1$ feasible actions, let $\Delta(H)$ be the gap between the best and second-best online action costs.
\begin{samepage}
\begin{theorem}[Decision-margin transfer]
\label{thm:decision_margin}
Assume the static evidence setting, integrable action costs and finite $v^2$. If $\Delta(H)>b(H)$ and
\begin{equation}
 M_T(H):=\frac{1-e^{-(\Delta(H)-b(H))/T_Q}}{n}
             -\frac{v^2(U)}{2T_Q^2}>2e(U),
 \label{eq:decision_margin}
\end{equation}
argmax student decoding selects the unique online-optimal action at $H$. If these conditions hold at every feasible multi-action history, the student is optimal for the fixed-output loss--cost objective.
\end{theorem}
\end{samepage}
Appendix~\ref{app:distillproofs} gives proofs, regret and pruning bounds, and an optimal-policy example with positive hindsight advantage.

%% file: sections/06_experiments.tex
\section{Experiments}
\label{sec:experiments}
We evaluate whether complete-record supervision yields a useful acquisition policy and examine the design choices behind its behavior. The evaluation proceeds from prediction quality and evidence cost (Section~\ref{sec:exp_main}), to the roles of representation, cost calibration, and the controller (Section~\ref{sec:exp_ablation}). We then test the separation of prediction from cost-sensitive planning by changing evidence prices while retaining the predictor (Section~\ref{sec:exp_prices}), before evaluating geographic and seasonal transfer (Section~\ref{sec:exp_transfer}).

\subsection{Experimental setup}
\label{sec:exp_setup}
\paragraph{Task and evidence.}
The retrospective influenza benchmark contains 5,933 state--week prediction episodes. Each record has an epidemic-phase label and a binary alert label. Wastewater measurements (WW) provide the initial evidence; queries reveal emergency-department visits (ED), hospital admissions (HOSP), laboratory positivity (POS), or source-context records (SRC). Their default costs are $(1.0,1.5,1.2,0.4)$, respectively, so one ED query defines one cost unit and the full optional bundle costs $4.1$. These are benchmark acquisition costs, not monetary prices. All queried evidence is restricted to information available by the corresponding state--week cutoff, and train/development/test partitions are chronological. We fit the predictor and controller on the training split and select calibration parameters and operating points on the development split.

\paragraph{Baselines.}
\emph{XGBoost (WW)} uses boosted trees \citep{chen2016}, while \emph{SelectiveNet (WW)} jointly models prediction and rejection \citep{geifman2019}; both use only the initial wastewater evidence. The multilayer perceptron (MLP) references \emph{MLP (WW+SRC)} and \emph{MLP (WW+ED)} add source context and emergency-department evidence, respectively. \emph{MLP (All)} uses WW and all optional groups.

For adaptive acquisition, \emph{EDDI} estimates expected information gain with a partial variational autoencoder \citep{ma2019}; \emph{GDFS} learns a greedy conditional-information selector \citep{covert2023}; and \emph{ACO} uses a nongreedy acquisition-conditioned oracle \citep{valancius2024}. \SLD{}, ACO, and GDFS share the same trained backbone, prediction head, evidence representations, data splits, and acquisition costs. The remaining controls are \emph{ED-First}, a fixed acquisition workflow beginning with ED; \emph{ReAct Router}, the benchmark's ReAct-style evidence-routing baseline \citep{yao2023}; and \emph{Supervised Router}, which learns routing actions through supervised training.

\paragraph{Evaluation.}
Phase macro-F1 and realized mean acquisition cost are the primary endpoints; expected calibration error (ECE) is complementary. Alert AUROC and Brier scores appear in Appendix~\ref{app:extended}; Appendix~\ref{app:implementation} gives the task-specific scoring and training protocol.

\subsection{Prediction quality and acquisition cost}
\label{sec:exp_main}
The first question is whether the complete pipeline improves prediction without simply collecting more evidence. 
% Table~\ref{tab:main} places no-query, fixed-workflow, and adaptive systems on the same reported prediction--cost comparison, using realized expenditures rather than nominal budget labels.
Table~\ref{tab:main} compares prediction quality and realized acquisition cost across no-query, fixed-workflow, and adaptive methods.

\begin{table}[!htbp]
\centering
\caption{\textbf{Prediction and evidence cost under the default price profile.} We report the mean and standard deviation over three seeds for realized acquisition cost, phase macro-F1, and ECE. MLP (WW+ED) and MLP (All) provide fixed-input references for the adaptive methods.}
\label{tab:main}
\small
\begin{tabular}{lrrr}
\toprule
System & Mean cost $\downarrow$ & Phase F1 $\uparrow$ & ECE $\downarrow$\\
\midrule
XGBoost (WW) & $0.000_{\pm 0.000}$ & $0.725_{\pm 0.001}$ & $0.031_{\pm 0.004}$\\
SelectiveNet (WW) & $0.000_{\pm 0.000}$ & $0.623_{\pm 0.008}$ & $0.197_{\pm 0.042}$\\
MLP (WW+SRC) & $0.400_{\pm 0.000}$ & $0.671_{\pm 0.004}$ & $0.025_{\pm 0.001}$\\
ReAct Router & $0.472_{\pm 0.031}$ & $0.622_{\pm 0.004}$ & $0.116_{\pm 0.002}$\\
MLP (WW+ED) & $1.000_{\pm 0.000}$ & $0.964_{\pm 0.003}$ & $0.073_{\pm 0.010}$\\
EDDI & $1.761_{\pm 0.071}$ & $0.814_{\pm 0.031}$ & $0.138_{\pm 0.026}$\\
GDFS & $2.187_{\pm 0.046}$ & $0.873_{\pm 0.011}$ & $0.130_{\pm 0.034}$\\
ACO & $2.357_{\pm 0.099}$ & $0.879_{\pm 0.006}$ & $0.131_{\pm 0.036}$\\
ED-First & $2.500_{\pm 0.000}$ & $0.831_{\pm 0.031}$ & $0.122_{\pm 0.018}$\\
Supervised Router & $3.017_{\pm 0.262}$ & $0.866_{\pm 0.014}$ & $0.126_{\pm 0.030}$\\
MLP (All) & $4.100_{\pm 0.000}$ & $0.936_{\pm 0.004}$ & $0.046_{\pm 0.006}$\\
\midrule
\textbf{\SLD\ (ours)} & \textbf{2.179}$_{\pm 0.027}$  & \textbf{0.983}$_{\pm 0.004}$ & \textbf{0.014}$_{\pm 0.001}$\\
\bottomrule
\end{tabular}
\end{table}

% \begin{table}[!htbp]
% \centering
% \caption{\textbf{Prediction and evidence cost under the default price profile.}
% Rows with repeated runs report mean $\pm$ standard deviation over three seeds.}
% \label{tab:main_3seed}
% \small
% \begin{tabular}{lccc}
% \toprule
% XGBoost (WW) & $0.000_{\pm 0.000}$ & $0.725_{\pm 0.001}$ & $0.031_{\pm 0.004}$\\
% SelectiveNet (WW) & $0.000_{\pm 0.000}$ & $0.623_{\pm 0.008}$ & $0.197_{\pm 0.042}$\\
% MLP (WW+SRC) & $0.400_{\pm 0.000}$ & $0.671_{\pm 0.004}$ & $0.025_{\pm 0.001}$\\
% ReAct Router & $0.472_{\pm 0.031}$ & $0.622_{\pm 0.004}$ & $0.116_{\pm 0.002}$\\
% MLP (WW+ED) & $1.000_{\pm 0.000}$ & $0.964_{\pm 0.003}$ & $0.073_{\pm 0.010}$\\
% EDDI & $1.761_{\pm 0.071}$ & $0.814_{\pm 0.031}$ & $0.138_{\pm 0.026}$\\
% GDFS & $2.187_{\pm 0.046}$ & $0.873_{\pm 0.011}$ & $0.130_{\pm 0.034}$\\
% ACO & $2.357_{\pm 0.099}$ & $0.879_{\pm 0.006}$ & $0.131_{\pm 0.036}$\\
% ED-First & $2.500_{\pm 0.000}$ & $0.831_{\pm 0.031}$ & $0.122_{\pm 0.018}$\\
% Supervised Router & $3.017_{\pm 0.262}$ & $0.866_{\pm 0.014}$ & $0.126_{\pm 0.030}$\\
% MLP (All) & $4.100_{\pm 0.000}$ & $0.936_{\pm 0.004}$ & $0.046_{\pm 0.006}$\\
% \midrule
% \textbf{\SLD\ (ours)} & \textbf{2.179}$_{\pm 0.027}$  & \textbf{0.983}$_{\pm 0.004}$ & \textbf{0.014}$_{\pm 0.001}$\\
% \bottomrule
% \end{tabular}
% \end{table}

\paragraph{Comparison with adaptive acquisition.}
\SLD{} achieves phase F1 $0.983$ at mean cost $2.179$. Compared with ACO, F1 increases by $0.104$ while cost decreases by $7.6\%$. GDFS likewise has lower F1 ($0.873$) at higher cost ($2.187$). \SLD{} also has lower ECE than ACO and GDFS. Relative to the fixed full-evidence MLP (All), \SLD{} improves phase F1 by $0.047$ while reducing mean acquisition cost by $46.9\%$. Thus, its phase-prediction gain is not explained by greater evidence consumption: it achieves higher F1 while acquiring substantially less evidence on average.

\paragraph{Comparison with a strong fixed workflow.}
MLP (WW+ED) reaches F1 $0.964$ at cost $1.000$, providing a strong low-cost alternative. \SLD{} adds $0.019$ F1 at $1.179$ additional cost; MLP (WW+ED) also has better alert AUROC/Brier ($0.975/0.059$ versus $0.966/0.068$). \SLD{} therefore offers higher phase F1, while MLP (WW+ED) remains a strong low-cost alert predictor. 
% The following study examines components and supervision designs.

%% file: sections/07_design_ablations.tex
\subsection{Ablations and supervision comparison}
\label{sec:exp_ablation}
The design separates subset evaluation, cost-sensitive planning, and online control. Table~\ref{tab:ablations} compares component removals, cost scaling, and oracle-trajectory supervision. We examine both prediction and acquisition cost, since each intervention can change the evidence used.

\begin{table}[!htbp]
\centering
\caption{\textbf{Ablations and supervision comparison.}}
\label{tab:ablations}
\small
\begin{tabular}{lrrr}
\toprule
Variant / method & Mean cost $\downarrow$ & Phase F1 $\uparrow$ & ECE $\downarrow$\\
\midrule
\textbf{\SLD\ (ours)} & \textbf{2.179} & \textbf{0.983} & \textbf{0.014}\\
W/O probabilistic state & 4.100 & 0.930 & 0.040\\
W/O frozen text adapter & 2.263 & 0.963 & 0.023\\
W/O policy network & 0.000 & 0.670 & 0.108\\
\midrule
Raw unaligned units & 0.000 & 0.150 & 0.032\\
\midrule
Oracle trajectories (ACO) & 2.357 & 0.879 & 0.131\\
\bottomrule
\end{tabular}
\end{table}

\paragraph{Subset representation.}
The probabilistic state supports both subset prediction and policy inputs. Removing it increases mean cost to $4.100$, the full-bundle cost, while F1 falls to $0.930$. The full model predicts more accurately with less evidence, supporting the state representation's role in the acquisition pipeline.

The frozen text adapter encodes contextual evidence. Removing it lowers F1 from $0.983$ to $0.963$ and raises cost from $2.179$ to $2.263$. Replacing E5-Mistral with gte-Qwen2 yields F1 $0.972$ at cost $2.134$ (Table~\ref{tab:encoders}). Both text encoders achieve higher phase F1 and lower mean acquisition cost than the variant without the text adapter.

\paragraph{Cost calibration and online control.}
The teacher compares stopping loss with continuation cost. Raw, unaligned units produce a no-query policy with F1 $0.150$. Removing the policy network also leaves a no-query system, with F1 $0.670$ at zero cost. These interventions connect score alignment and deployable control to the pipeline's use of additional evidence.

\paragraph{Oracle trajectories and uniform mask supervision.}
ACO's behavioral-cloning formulation learns from oracle rollouts \citep{valancius2024}. SLD instead repeatedly samples masks uniformly from each record's target lattice. Compared with ACO, SLD has higher F1 ($0.983$ versus $0.879$), lower cost ($2.179$ versus $2.357$), and lower ECE ($0.014$ versus $0.131$). 
% Together with the component study, this supports the effectiveness of the complete-record learning pipeline.

%% file: sections/08_replanning_transfer.tex
\subsection{Replanning under new costs with a fixed predictor}
\label{sec:exp_prices}
% To test replanning with a fixed predictor, each persistent price profile uses development-set multiplier recalibration, regenerated training/development teacher targets, and student redistillation. 
For each price profile, we keep the predictor and terminal procedure fixed, recalibrate the acquisition-cost multiplier on the development set, regenerate the teacher targets, and redistill the student. We do not use test labels for any of these steps.
The predictor and terminal procedure stay fixed; test outcomes are unused. Six profiles vary ED prices; three alter HOSP or SRC costs or swap the full cost ordering.

Across the nine nondefault profiles, $25.7\%$--$63.6\%$ of actions change while phase F1 remains in $[0.981,0.987]$ (Appendix~\ref{app:prices}). The full cost-rank swap changes $36.2\%$ of actions with F1 $0.985$. These results support price-responsive replanning while preserving prediction quality across the tested costs, without refitting the predictor. Adaptation uses redistillation, not zero-shot response.

\subsection{Transfer across geography and season}
\label{sec:exp_transfer}
The price study changes the acquisition objective while retaining the prediction model. We next ask whether the complete system remains useful when the evaluation population changes. HHS Region 3 is excluded from model fitting, threshold selection, cost calibration, and operating-point selection. We evaluate the frozen systems on 305 episodes from the held-out 2024/25 season. All evaluated systems are frozen for this holdout.

\begin{table}[!htbp]
\centering
\caption{\textbf{Frozen geographic and seasonal transfer.} Results on HHS Region 3 in 2024/25, excluded from fitting and model selection.}
\label{tab:transfer}
\small
\begin{tabular}{lrr}
\toprule
System & Mean cost $\downarrow$ & Phase F1 $\uparrow$\\
\midrule
ACO & 2.975 & 0.928\\
MLP (All) & 4.100 & 0.914\\
\midrule
\textbf{\SLD\ (ours)} & \textbf{2.700} & \textbf{0.958}\\
\bottomrule
\end{tabular}
\end{table}

\SLD{} reaches F1 $0.958$ at cost $2.700$, improving F1 by $0.030$ over ACO while using $0.275$ less cost (Table~\ref{tab:transfer}). Against MLP (All), the corresponding differences are $0.044$ higher F1 and $1.400$ lower cost. The higher-F1/lower-cost pattern relative to the full-evidence reference therefore also persists under the tested geographic and seasonal shift.

%% file: sections/09_conclusion.tex
\section{Conclusion}
\SLD{} converts complete records into acquisition policies through shared subset evaluation, cost-sensitive lattice planning, and observed-history distillation. Uniformly resampled masks train the controller across partial observations while combinatorial planning remains offline. Reusable subset scores support replanning with a fixed predictor. Our analysis identifies when soft-target distillation preserves online action ordering and how student approximation affects decision quality.

On retrospective influenza monitoring, \SLD{} improves phase prediction over adaptive acquisition baselines at lower acquisition cost. The geographic and seasonal holdout also shows higher F1 and lower cost than ACO and MLP (All). 
% Component, supervision-reference, and price studies characterize the learned acquisition pipeline. 
Sparse subset search provides a direct computational extension to larger acquisition spaces, with pruning error linked to decision stability through the same distillation interface.

%% file: sections/statements.tex
\section*{Reproducibility Statement}
The appendix supplies proofs, implementation details, and additional experimental results. 
Code will accompany the submission.

\section*{AI Use Statement}
Generative AI tools were used for assistance with manuscript editing and code-related tasks. All scientific contributions, experimental design, analysis, and conclusions were developed and verified by the authors.
% Authors remain responsible for verifying the mathematics, sources, implementation, and statements before submission.
\newpage

%% file: sections/appendix.tex
\input{sections/appendix_a_distillation}
\input{sections/appendix_b_benchmarks}
\input{sections/appendix_c_results}

\input{sections/appendix_d_protocol}

%% file: sections/appendix_a_distillation.tex
\section{Soft-Target Distillation: Analysis and Proofs}
\label{app:distillproofs}
This appendix analyzes the soft action targets and policy cross-entropy of \SLD. Conditional averaging is a standard property of asymmetric imitation \citep{warrington2021}; the pathwise cost argument is related to cost-sensitive imitation analysis \citep{ross2014}. We specialize these tools to an exhaustive, static evidence lattice with outcome-dependent terminal scores. The analysis does not alter the teacher, train a value regressor, or change the distillation objective.

\subsection{Population setting and comparison class}
\label{app:population}
Condition on all fitted functions: the predictor, score map, representation, student, and deployed terminal rule $g$. Fix prices $c=(c_m)_{m\in\M}$, the multiplier $\lambda$, and target temperature $T_Q>0$. A fresh complete record $W$ has distribution $P$ and contains the mandatory evidence, optional group values, metadata, and outcome. It is independent of the randomness used to fit these functions. All expectations and conditional distributions below refer to this same $P$. This is a population analysis of fitted functions, not a finite-sample generalization theorem.

For each mask $S\subseteq\M$, let $H_S(W)$ be the corresponding observed history and let $U_S=s(H_S)$ be the actual student input. It includes $S$ and any availability information needed to determine the action set. Each group is queried at most once, queries only reveal stored values, and acquisition does not change the underlying record. The analyzed action set is $\A_S^{\rm acq}=\{\stopa\}\cup(\M\setminus S)$. A fixed observed restriction may reduce it, provided the teacher, student, and comparison policies use the same feasible actions. The full mask always stops.

Independently of $W$, sample $S\sim w=\operatorname{Unif}(2^\M)$, matching the repeated uniform mask sampling used for student optimization. Thus $w_S=2^{-K}$, $\nu=P\times w$, and $w_{\min}=2^{-K}$. Uniformity is over individual subsets, not over subset sizes. The proof below is written for general positive weights $w_S$, with the uniform scheme as the implementation used here. Repeated random masks estimate the all-mask average without requiring every mask in each update; finite runs need not visit every record--mask pair.

The fixed rule $g$ is measurable with respect to the observed history. It specifies prediction and, when used, whether to defer. Write $L_g(W,S)$ for its realized loss; independent terminal randomization can be integrated into this loss. Define
\begin{equation}
 J_{\lambda,g}(\pi)=\E\left[L_g(W,S_T)+\lambda\sum_{m\in S_T}c_m\right],\qquad
 J_{\lambda,g}^{\star}=\inf_{\pi\in\Pi_{\rm obs}}J_{\lambda,g}(\pi),
 \label{eq:fixedg}
\end{equation}
where $\Pi_{\rm obs}$ consists of nonanticipative acquisition policies using the same predictor, terminal rule, costs, and feasible actions. These policies may use the full observed history, not only its compressed representation. They do not observe the outcome or unqueried values. Thus the comparator can choose a better acquisition policy but cannot replace the terminal predictor or rejection procedure.

Let $\rho_W(S)$ be the realized score supplied to the teacher and define
\begin{equation}
 J_{\rho}(\pi)=\E\left[\rho_W(S_T)+\lambda\sum_{m\in S_T}c_m\right].
 \label{eq:rhoobjective}
\end{equation}
Losses and scores are integrable. For the global KL bound, assume Equation~\eqref{eq:terminal_mismatch} for each mask and almost every history. The parameter $\delta$ bounds the conditional-mean discrepancy between the teacher score $\rho_W(S)$ and the loss $L_g(W,S)$ of the fixed terminal rule, including effects of calibration, subset regularization, and clipping. The decision-margin result instead uses the action-relative distortion in Equation~\eqref{eq:decision_quantities}. Appendix~\ref{app:surrogate} analyzes terminal-score perturbations and nonlinear clipping.

For notational brevity, suppress the fixed $\lambda$ in $\widetilde V_W(S)$ and $\widetilde Q_W(a,S)$. These are precisely the realized teacher values from Section~\ref{sec:teacher}. For the global KL bound, assume
\begin{equation}
 \max_{a\in\A_S^{\rm acq}}\widetilde Q_W(a,S)
 -\min_{a\in\A_S^{\rm acq}}\widetilde Q_W(a,S)\leq B_Q
 \label{eq:qspan}
\end{equation}
for every mask and almost every record, with $B_Q>0$. For example, scores in $[0,B_\rho]$ and finite prices imply $0\leq\widetilde V_W(S)\leq B_\rho$ and permit $B_Q=B_\rho+\lambda\max_m c_m$. The span assumption is additional to the algorithm; boundedness of an arbitrary untransformed log loss is not presumed.

\subsection{Proof of the conditional projection}
\label{app:projection}
For any student with finite expected divergence, define
\begin{equation}
 \mathcal E_{\rm priv}=\E_\nu\KL(p_T\Vert\bar p),\qquad
 \mathcal E_{\rm fit}=\E_\nu\KL(\bar p\Vert\pi_\eta).
 \label{eq:errors}
\end{equation}
Then $\E_\nu\KL(p_T\Vert\pi_\eta)=\mathcal E_{\rm priv}+\mathcal E_{\rm fit}$.

For each feasible action, define $p_W(a,S)=p_T(a\mid W,S)$ and $\bar p(a\mid U)=\E_\nu[p_W(a,S)\mid U]$. Since the mask is included in $U$, all distributions in a conditional average have the same action set. For any observed-input distribution $\pi$,
\begin{align}
 \E_\nu\left[-\sum_a p_W(a,S)\log\pi(a\mid U)\right]
 &=\E_\nu\left[-\sum_a\bar p(a\mid U)\log\pi(a\mid U)\right]\nonumber\\
 &=\E_\nu\left[\mathsf H(\bar p(\cdot\mid U))+\KL(\bar p\Vert\pi)\right],
 \label{eq:projectionproof}
\end{align}
where $\mathsf H$ is Shannon entropy with natural logarithms. Nonnegativity of KL proves Equation~\eqref{eq:projection}. The minimizer is unique up to null sets over unrestricted distributions. A finite network or shared training with the terminal head need not attain it.

For finite expected divergences, add and subtract $\log\bar p$ inside the KL and condition its second term on $U$:
\begin{align}
 \E_\nu\KL(p_W\Vert\pi)
 &=\E_\nu\KL(p_W\Vert\bar p)
   +\E_\nu\sum_a p_W(a,S)\log\frac{\bar p(a\mid U)}{\pi(a\mid U)}\nonumber\\
 &=\E_\nu\KL(p_W\Vert\bar p)+\E_\nu\KL(\bar p\Vert\pi).
 \label{eq:projectiondecomp}
\end{align}
This proves Proposition~\ref{prop:projection}. With finite teacher costs and $T_Q>0$, teacher probabilities are positive on feasible actions. A student assigning zero probability where the teacher is positive has infinite KL and a vacuous bound, rather than an exception to the statement.

$\mathcal E_{\rm priv}$ is the irreducible discrepancy for a \emph{fixed} teacher and student input: it measures variation of preferences across records indistinguishable through that input. A richer representation can change it, but enumerating additional masks does not reveal hidden values at deployment. It can also be positive when action rankings agree, so small $\mathcal E_{\rm priv}$ is a sufficient, not necessary, condition for good decisions. The projection averages probabilities, not costs:
\begin{equation}
 \E_\nu[\operatorname{softmax}(-\widetilde Q_W/T_Q)\mid U]
 \ne \operatorname{softmax}(-\E_\nu[\widetilde Q_W\mid U]/T_Q)
 \quad\text{in general}.
 \label{eq:noncommute}
\end{equation}

\input{sections/appendix_a_decision_stability}

\subsection{Global soft-target transfer bound}
Fix the deployed terminal rule $g$, including any deferral. Let $L_g(W,S)$ be its loss at $S$, $J_{\lambda,g}$ the objective in Equation~\eqref{eq:objective}, and $J_{\lambda,g}^{\star}$ its infimum over observed-history acquisition policies using the \emph{same} predictor and rule. For $\delta\geq0$, assume
\begin{equation}
 \left|\E[\rho_W(S)-L_g(W,S)\mid H_S]\right|\leq\delta
 \label{eq:terminal_mismatch}
\end{equation}
for every mask and almost every history; realized scores need not equal conditional risks.
\begin{theorem}[Soft-target transfer bound]
\label{thm:transfer}
Assume the static evidence setting of Section~\ref{sec:problem}, integrable losses, Equation~\eqref{eq:terminal_mismatch}, and teacher action-cost span at most $B_Q>0$. Set $w_{\min}=2^{-K}$. For $\widehat\pi(U)\in\argmax_a\pi_\eta(a\mid U)$ and $T_Q>0$,
\begin{equation}
 0\leq J_{\lambda,g}(\widehat\pi)-J_{\lambda,g}^{\star}
 \leq C\sqrt{\mathcal E_{\rm priv}+\mathcal E_{\rm fit}}+2\delta.
 \label{eq:transfer}
\end{equation}
The constant $C$ depends only on $K,T_Q,B_Q,w_{\min}$ and is given below.
\end{theorem}

\subsection{From soft preference error to path cost}
\label{app:pathcost}
Define the record-level disadvantage
\begin{equation}
 \Gamma_W(S,a)=\widetilde Q_W(a,S)-\widetilde V_W(S)\geq0.
 \label{eq:disadvantage}
\end{equation}
For a realized acquisition path, $\mathcal P$ denotes its sequence of $(S,a)$ pairs, including the final stop. If $S_T$ is its terminal subset, set $\mathcal C_\rho(W,\mathcal P)=\rho_W(S_T)+\lambda\sum_{m\in S_T}c_m$.
\begin{lemma}[Pathwise telescoping]
\label{lem:telescope}
Every feasible path satisfies
\begin{equation}
 \mathcal C_\rho(W,\mathcal P)-\widetilde V_W(\varnothing)
 =\sum_{(S,a)\in\mathcal P}\Gamma_W(S,a).
 \label{eq:telescope}
\end{equation}
\end{lemma}
\begin{proof}
For a query $m$, the summand is $\lambda c_m+\widetilde V_W(S\cup\{m\})-\widetilde V_W(S)$. For the final stop it is $\rho_W(S_T)-\widetilde V_W(S_T)$. Consecutive values cancel. The remaining terms are the query costs, terminal score, and negative initial teacher value. Each summand is nonnegative because $\widetilde V_W(S)$ is the minimum action cost.
\end{proof}

The next lemma translates the actual soft-target loss into a cost disadvantage for deterministic decoding. It does not replace the training loss with cost regression.
\begin{lemma}[Softmax ranking control]
\label{lem:ranking}
Let $q\in\mathbb R^n$ have span at most $B_Q$, $p_a\propto e^{-q_a/T_Q}$, and let $\widehat a$ maximize any probability vector $v$. Then
\begin{equation}
 q_{\widehat a}-\min_a q_a
 \leq \frac{nB_Q}{1-e^{-B_Q/T_Q}}\|p-v\|_1
 \leq \frac{nB_Q}{1-e^{-B_Q/T_Q}}\sqrt{2\KL(p\Vert v)}.
 \label{eq:ranking}
\end{equation}
\end{lemma}
\begin{proof}
Choose a minimum-cost action $a^*$ and put $d=q_{\widehat a}-q_{a^*}\in[0,B_Q]$. Then $p_{a^*}\geq1/n$ and
\[
 p_{a^*}-p_{\widehat a}=p_{a^*}(1-e^{-d/T_Q})
 \geq \frac{d}{nB_Q}(1-e^{-B_Q/T_Q}).
\]
The last inequality uses concavity of $1-e^{-d/T_Q}$ on $[0,B_Q]$. Since $v_{\widehat a}\geq v_{a^*}$, the same probability difference is at most $|p_{a^*}-v_{a^*}|+|p_{\widehat a}-v_{\widehat a}|\leq\|p-v\|_1$. This proves the first inequality; Pinsker's inequality proves the second. If $\widehat a=a^*$ the claim is immediate, so duplicated indices do not enter the argument. Arbitrary observed-input tie-breaking is allowed.
\end{proof}

\subsection{Proof of the soft-target transfer bound}
\label{app:transferproof}
Write $h=K+1$ and $A_Q=hB_Q/(1-e^{-B_Q/T_Q})$. Let $I_S(W)$ indicate that the argmax student's path visits $S$, and let $k_W(S)=\KL(p_W(\cdot,S)\Vert\pi_\eta(\cdot\mid U_S))$. Each mask is visited at most once and $\sum_S I_S\leq h$. Lemmas~\ref{lem:telescope}--\ref{lem:ranking} imply
\begin{align}
 J_\rho(\widehat\pi)-\E\widetilde V_W(\varnothing)
 &\leq A_Q\sqrt{2}\,\E\sum_S I_S\sqrt{k_W(S)}\nonumber\\
 &\leq A_Q\sqrt{2h\,\E\sum_S I_S k_W(S)}\nonumber\\
 &\leq A_Q\sqrt{\frac{2h}{w_{\min}}\E_\nu k_W(S)}.
 \label{eq:occupancybound}
\end{align}
The second line applies Cauchy--Schwarz along each path and then Jensen's inequality. The last line follows pointwise from $0\leq I_S\leq1$ and $w_S\geq w_{\min}$. Importantly, the visited masks may depend on observed values: no independence between visitation and teacher errors is assumed. Static revelation makes all counterfactual masks of the \emph{same} fresh record available for this domination argument.

The complete-record teacher minimizes realized score-plus-cost over all feasible paths. Consequently,
\begin{equation}
 \E\widetilde V_W(\varnothing)\leq\inf_{\pi\in\Pi_{\rm obs}}J_\rho(\pi).
 \label{eq:relaxedlower}
\end{equation}
It remains to connect score cost to the actual fixed terminal procedure. For a deterministic nonanticipative acquisition policy, the event $\{S_T=S\}$ is measurable from $H_S$: deciding whether a path ends at $S$ only requires values in that final subset. For randomized policies the same statement holds conditional on independent policy randomness. Therefore,
\begin{align}
 J_\rho(\pi)-J_{\lambda,g}(\pi)
 &=\sum_S\E\left[\ind\{S_T=S\}\,\E[\rho_W(S)-L_g(W,S)\mid H_S]\right],
 \label{eq:terminalbiasproof}
\end{align}
with independent randomization integrated out. Equation~\eqref{eq:terminal_mismatch} and $\sum_S\ind\{S_T=S\}=1$ give $|J_\rho(\pi)-J_{\lambda,g}(\pi)|\leq\delta$ for every comparison policy. Only one terminal loss is paid, so this error does not multiply by the number of queries.

Combining Equations~\eqref{eq:occupancybound}--\eqref{eq:terminalbiasproof} and Proposition~\ref{prop:projection},
\begin{align}
 J_{\lambda,g}(\widehat\pi)-J_{\lambda,g}^{\star}
 &\leq J_\rho(\widehat\pi)-\inf_{\pi\in\Pi_{\rm obs}}J_\rho(\pi)+2\delta\nonumber\\
 &\leq \frac{(K+1)B_Q}{1-e^{-B_Q/T_Q}}
     \sqrt{\frac{2(K+1)}{w_{\min}}
       (\mathcal E_{\rm priv}+\mathcal E_{\rm fit})}+2\delta.
 \label{eq:fulltransfer}
\end{align}
The lower bound is the definition of the infimum, since $\widehat\pi\in\Pi_{\rm obs}$. This proves Theorem~\ref{thm:transfer}. For the uniform scheme, substitute $w_{\min}=2^{-K}$; the factor inside the square root is $2(K+1)2^K$.

The global bound controls path error through the uniform-mask population loss and therefore carries a coverage factor. Theorem~\ref{thm:decision_margin} instead uses conditional action margins and approximation errors at observed inputs. These are different sufficient conditions: increasing temperature reduces nonlinear aggregation error in the latter, but also narrows probability margins. Both analyses refer to the specified population $P$.

\subsection{Action compatibility and improvement over a reference}
A useful exact case does not require reconstructing unqueried measurements. Suppose there is an observed-input action $a^\dagger(U)$ that minimizes $\widetilde Q_W(\cdot,S)$ for almost every record consistent with $U$, at every mask. Then it maximizes every corresponding soft teacher distribution. It therefore maximizes $\bar p$; any other maximizer of $\bar p$ is also a teacher maximizer almost surely, since the nonnegative teacher-probability gap has mean zero. An exact projection student with argmax decoding incurs zero disadvantage on each visited mask. Lemma~\ref{lem:telescope} gives $J_\rho(\widehat\pi)=\E\widetilde V_W(\varnothing)$, so it is online-optimal for the score objective. With $\delta=0$ it is optimal for $J_{\lambda,g}$ as well. This sufficient compatibility condition concerns action preferences, not equality of all hidden evidence values.

More generally, let $\pi_b$ be any reference acquisition policy using the same fixed output rule and define $\Delta_b=J_{\lambda,g}(\pi_b)-J_{\lambda,g}^{\star}$. If the right side of Equation~\eqref{eq:transfer} is strictly smaller than $\Delta_b$, the distilled controller improves on that reference in the loss--cost objective. This implication compares acquisition policies for the same fixed output rule and population.

\subsection{When actions are sampled rather than decoded by argmax}
\label{app:sampling}
The decision-margin theorem and Theorem~\ref{thm:transfer} apply to deterministic argmax decoding. The same learned distribution can be analyzed under sampling without conflating the two deployment rules.
\begin{corollary}[Sampled-policy bound]
\label{cor:sampling}
Under the assumptions of Theorem~\ref{thm:transfer}, sampling $a\sim\pi_\eta(\cdot\mid U)$ independently at each visited state yields
\begin{equation}
 J_{\lambda,g}(\pi_\eta)-J_{\lambda,g}^{\star}
 \leq h\min\{B_Q,T_Q\log h\}
 +B_Q\sqrt{\frac{h}{2w_{\min}}
       (\mathcal E_{\rm priv}+\mathcal E_{\rm fit})}+2\delta,
 \quad h=K+1.
 \label{eq:samplingbound}
\end{equation}
\end{corollary}
\begin{proof}
At a fixed record and mask, shift the teacher costs by their minimum to obtain disadvantages $d_a\in[0,B_Q]$. For $p_a=e^{-d_a/T_Q}/Z$ with $Z\geq1$,
\[
 \E_{a\sim p}d_a=T_Q\bigl(\mathsf H(p)-\log Z\bigr)
 \leq\min\{B_Q,T_Q\log|\A_S^{\rm acq}|\}.
\]
Changing the action distribution from $p$ to $\pi_\eta$ changes this expectation by at most $B_Q\operatorname{TV}(p,\pi_\eta)\leq B_Q\sqrt{\KL(p\Vert\pi_\eta)/2}$. Apply the pathwise identity conditional on independent action randomness, sum at most $h$ entropy terms, and use the same coverage and terminal-bias arguments as above. This proves Equation~\eqref{eq:samplingbound}.
\end{proof}
The additional temperature term accounts for sampling suboptimal actions, even when the student matches the soft teacher. Argmax and sampled decoding therefore have distinct guarantees.

%% file: sections/appendix_a_decision_stability.tex
\subsection{Decision margins and conditional target averaging}
\label{app:decisionproof}
The projection in Proposition~\ref{prop:projection} averages probabilities rather than action costs. This subsection bounds the resulting change in action ordering. It keeps the same realized teacher and policy loss; none of the conditional quantities below is an additional fitted module.

For the fixed predictor and deployed terminal rule $g$, define
\begin{equation}
\begin{split}
 R_g(H)&=\E[L_g(W,S)\mid H_S=H],\\
 Q_g^*(H,\stopa)&=R_g(H),\\
 Q_g^*(H,m)&=\lambda c_m+\E[V_g^*(H_m^+)\mid H],\\
 V_g^*(H)&=\min_{a\in\A_S^{\rm acq}}Q_g^*(H,a).
\end{split}
\label{eq:actual_bellman}
\end{equation}
The finite, acyclic acquisition process determines these values by backward induction, with only stop available at a full mask. The terminal rule is fixed, including any prediction/deferral decision; the comparison does not give the reference policy a different predictor or terminal head. All conditional expectations refer to the fresh-record distribution in Appendix~\ref{app:population}.

At input $U=u$, write $q(u)=\bar q(u)$ and $Z=\widetilde Q_W-q(u)$, so that $\E[Z\mid U=u]=0$. Teacher and student action sets agree because $U$ contains the mask and observed availability information. Define $p^0(u)=\operatorname{softmax}(-q(u)/T_Q)$. The actual projected target $\bar p$ need not equal $p^0$.
\begin{lemma}[Averaging error for pairwise preferences]
\label{lem:soft_averaging}
For any two feasible actions $i,j$ and finite $v^2(u)$,
\begin{equation}
 \left|\bar p_i(u)-\bar p_j(u)-\big(p_i^0(u)-p_j^0(u)\big)\right|
 \leq\frac{v^2(u)}{2T_Q^2}.
 \label{eq:pairwise_averaging}
\end{equation}
\end{lemma}
\begin{proof}
Fix a cost perturbation $z$ and put $p_i(s)=\operatorname{softmax}_i(-(q+sz)/T_Q)$ for $s\in[0,1]$. Let $m(s)=\sum_a p_a(s)z_a$. Direct differentiation gives
\begin{equation}
 p_i''(s)=\frac{p_i(s)}{T_Q^2}
 \left[(z_i-m(s))^2-\operatorname{Var}_{p(s)}(z)\right].
 \label{eq:soft_hessian}
\end{equation}
Both terms in brackets are nonnegative before subtraction and at most $\operatorname{osc}(z)^2$. Hence
$|p_i''(s)-p_j''(s)|\leq(p_i(s)+p_j(s))\operatorname{osc}(z)^2/T_Q^2\leq\operatorname{osc}(z)^2/T_Q^2$.
Taylor's formula with integral remainder bounds the pairwise remainder by $\operatorname{osc}(z)^2/(2T_Q^2)$. The first derivative at zero is linear in $z$, so its conditional expectation vanishes. Taking expectations proves the claim. Finite conditional first moments and $v^2(u)$ justify these operations.
\end{proof}

\paragraph{Proof of Theorem~\ref{thm:decision_margin}.}
Let $a^*$ be the unique minimizer of $Q_g^*(H,\cdot)$. By the definition of $b(H)$, every competitor $a\ne a^*$ satisfies
\begin{equation}
 \bar q_a(U)-\bar q_{a^*}(U)
 \geq Q_g^*(H,a)-Q_g^*(H,a^*)-b(H)
 \geq\Delta(H)-b(H)>0.
 \label{eq:mean_gap}
\end{equation}
Thus $a^*$ also minimizes the conditional mean teacher costs. With $n$ feasible actions, $p^0_{a^*}\geq1/n$, and
\begin{equation}
 p^0_{a^*}-p^0_a
 =p^0_{a^*}\left[1-e^{-(\bar q_a-\bar q_{a^*})/T_Q}\right]
 \geq\frac{1-e^{-(\Delta-b)/T_Q}}{n}.
\end{equation}
Lemma~\ref{lem:soft_averaging} implies $\bar p_{a^*}-\bar p_a\geq M_T(H)$. Student approximation can decrease a pairwise probability gap by at most $2e(U)$, so the strict criterion in Equation~\eqref{eq:decision_margin} makes $a^*$ the unique student maximizer. When the criterion holds at every multi-action history, backward induction identifies the resulting policy with an optimal observed-history acquisition policy. Single-action histories are automatic. This proves the theorem.

\paragraph{From local decisions to policy cost.}
For a history with multiple online-optimal actions, the strict criterion need not certify any particular tie. Define $\mathcal B_\pi$ as histories where $\Delta\leq b$ or $M_T\leq2e$, including such ties, and let $B_g(H)=\operatorname{osc}(Q_g^*(H,\cdot))$. At a single-action history take $B_g(H)=0$. The finite-horizon performance-difference identity gives
\begin{equation}
\begin{split}
 J_{\lambda,g}(\widehat\pi)-J_{\lambda,g}^{\star}
 &=\E_{\widehat\pi}\sum_{t=0}^{T}
       \left[Q_g^*(H_t,A_t)-V_g^*(H_t)\right]\\
 &\leq\E_{\widehat\pi}\sum_{t=0}^{T}
       B_g(H_t)\ind\{H_t\in\mathcal B_\pi\}.
\end{split}
\label{eq:margin_regret}
\end{equation}
Here $T$ is the final acquisition-decision index and $A_T=\stopa$. The equality follows by substituting Equation~\eqref{eq:actual_bellman} and telescoping conditional successor values in expectation; the terminal step contributes $R_g(H_T)$, whose expectation equals the actual terminal loss. The upper bound uses Theorem~\ref{thm:decision_margin} outside $\mathcal B_\pi$ and the action span inside it. If $B_g(H)\leq B_*$, this is at most $B_*$ times the expected number of uncertified decisions, at most $(K+1)B_*$ overall. No uniform-mask coverage factor appears because the premise is a conditional, input-specific criterion, not an average training-error bound. It is not a polynomial sample-complexity guarantee.

\paragraph{Temperature and a nonempty sufficient-condition regime.}
Suppose uniformly $\Delta-b\geq d_0>0$, $v^2\leq V_0<\infty$, and $n\leq n_{\max}$. For $T_Q\geq d_0$, the elementary inequality $1-e^{-x}\geq x/2$ for $0\leq x\leq1$ gives
\begin{equation}
 M_T(H)\geq\frac{d_0}{2n_{\max}T_Q}-\frac{V_0}{2T_Q^2}.
\end{equation}
The right side is positive when $T_Q>n_{\max}V_0/d_0$. Any student with error below half this bound selects optimal actions. Thus finite fluctuations need not prevent consistency when mean action ordering is compatible. Increasing temperature also narrows the probability margin and requires more accurate fitting.

\subsection{Why absolute teacher variability need not create regret}
\label{app:common_shift}
\begin{corollary}[Common-score-shift invariance]
\label{cor:shift}
Replace every terminal score of a record by $\rho'_W(S)=\rho_W(S)+\zeta(W)$, where the same integrable $\zeta(W)$ is used at every subset. Teacher soft targets, their projection, $b(H)$, $v^2(U)$, and the decision criterion are unchanged.
\end{corollary}
\begin{proof}
Backward induction in Equation~\eqref{eq:teacher} gives $\widetilde V'_W(S)=\widetilde V_W(S)+\zeta(W)$ and $\widetilde Q'_W(a,S)=\widetilde Q_W(a,S)+\zeta(W)$ for every action. Softmax removes this common additive term. Conditional mean costs shift by the action-independent $\E[\zeta\mid U]$, leaving $b$ unchanged. The centered cost vector shifts by $(\zeta-\E[\zeta\mid U])\mathbf 1$, leaving its oscillation and hence $v^2$ unchanged.
\end{proof}
For a large positive constant $\zeta$, the absolute terminal mismatch used by the global bound can be large while every acquisition decision remains identical. Action-dependent score changes, including arbitrary nonlinear clipping, need not have this invariance; their effects enter $b$ and $v^2$.

\begin{proposition}[Symmetric binary preferences]
\label{prop:binary_symmetry}
With two feasible actions $a,b$, put $D=\widetilde Q_W(b,S)-\widetilde Q_W(a,S)$ and $\mu(U)=\E[D\mid U]$. If the conditional distribution of $D-\mu(U)$ is symmetric about zero, then, for every finite $T_Q>0$,
\begin{equation}
 \operatorname{sign}\!\left(\bar p(a\mid U)-\tfrac12\right)
 =\operatorname{sign}(\mu(U)).
 \label{eq:symmetry}
\end{equation}
Thus an exact projection student ranks actions by their conditional mean teacher costs. This is online-optimal whenever that ranking agrees with $Q_g^*(H,\cdot)$.
\end{proposition}
\begin{proof}
The teacher probability of $a$ is $\sigma(D/T_Q)$, where $\sigma(x)=(1+e^{-x})^{-1}$. If $Z=D-\mu$ is symmetric, $\E[\sigma(Z/T_Q)\mid U]=1/2$ by $\sigma(x)+\sigma(-x)=1$. The expectation of $\sigma((x+Z)/T_Q)$ is strictly increasing in $x$ for finite $Z$ and $T_Q>0$. Its sign relative to $1/2$ is therefore the sign of $\mu$.
\end{proof}

\paragraph{A positive information gap with an optimal projected policy.}
Consider a balanced binary target $Y$, and an independent success indicator $Z$ with probability $1/2$. Initially the fixed terminal rule defers at loss $0.3$. A query costs $0.1$ and reveals $(Z,Y)$ on success or a failure symbol otherwise. After a successful query the terminal rule predicts $Y$ with loss zero; on failure it defers at loss $0.3$. Use these exact realized losses as $\rho_W(S)$.

The initial teacher cost vector $(\widetilde Q_{\stopa},\widetilde Q_{\rm query})$ is $(0.3,0.1)$ on success and $(0.3,0.4)$ on failure. The online-optimal query costs $0.1+0.5(0.3)=0.25$, whereas immediate stopping costs $0.3$. The hindsight teacher attains $0.5(0.1)+0.5(0.3)=0.20$, giving information advantage $0.05$. Nevertheless,
\begin{equation}
 \bar p({\rm query})
 =\tfrac12\sigma(0.2/T_Q)+\tfrac12\sigma(-0.1/T_Q)>\tfrac12
 \qquad(T_Q>0).
\label{eq:positive_gap_example}
\end{equation}
The exact argmax projection therefore queries and has zero online regret for every finite positive temperature. Teacher preferences still vary across successful and failed records, so their conditional KL variability is positive. This analytical construction shows that a positive information-relaxation gap is compatible with optimal deployed decisions. The asymmetric counterexample in Appendix~\ref{app:examples} shows why compatibility cannot be presumed for all distributions.

\subsection{Sparse subset planning and decision stability}
\label{app:sparse}
Sparse planning changes the offline search budget, not the observed-history deployment interface. For a supervised root $S$, let $\mathcal G_{W,S}=(\mathcal V,\mathcal E)$ be a finite DAG of subsets containing $S$, with edges $R\to R\cup\{m\}$. Every retained node permits stopping at its original score $\rho_W(R)$. At the supervised root, include every feasible first query and its successor; deeper continuations can be pruned. The recurrence on this graph is
\begin{equation}
 \widetilde V^{\mathcal G}_W(R)=
 \min\!\left\{\rho_W(R),\min_{(R,R\cup\{m\})\in\mathcal E}
       [\lambda c_m+\widetilde V^{\mathcal G}_W(R\cup\{m\})]\right\}.
\label{eq:sparse_recurrence}
\end{equation}
A leaf stops; the inner minimum of an empty set is $+\infty$. This specifies an exact backward pass on the retained DAG. At the root it yields finite costs for all original actions:
$\widetilde Q_W^{\mathcal G}(\stopa,S)=\rho_W(S)$ and
$\widetilde Q_W^{\mathcal G}(m,S)=\lambda c_m+\widetilde V_W^{\mathcal G}(S\cup\{m\})$.
Every retained continuation is feasible in the full lattice, so $\widetilde Q_W^{\mathcal G}(a,S)\geq\widetilde Q_W(a,S)$. Equality is recovered when the graph contains a full-lattice optimal continuation for that first action. Merely retaining the optimal path from the initial empty mask need not give equality for other roots or actions.

\paragraph{Computation.}
For constructed graphs indexed by $j$ (record--root pairs), scoring takes at most $\sum_j|\mathcal V_j|$ predictor evaluations, and each price profile requires $O(\sum_j|\mathcal V_j|+\sum_j|\mathcal E_j|)$ scalar operations. Shared nodes can be cached. These are post-construction costs; proposal generation and scoring discarded candidates also consume work. A width-$b$, depth-$D$ beam that expands at most $K$ children per retained node generates at most $1+K+(D-1)bK$ candidate states per supervised root, before duplicates are removed. Thus its candidate-scoring count is $O(bKD)$, with separate ranking/bookkeeping costs. Stopping at every leaf defines the truncation without a learned value regressor. Search budgets bound computation, while discarded continuations determine approximation error.

\begin{lemma}[Soft-target sensitivity to pruning]
\label{lem:sparse_softmax}
For cost vectors $q,q'$ on the same action set and $p=\operatorname{softmax}(-q/T_Q)$, $p'=\operatorname{softmax}(-q'/T_Q)$,
\begin{equation}
 \|p'-p\|_1\leq\frac{\operatorname{osc}(q'-q)}{2T_Q}.
 \label{eq:sparse_lipschitz}
\end{equation}
\end{lemma}
\begin{proof}
Write $d=q'-q$ and $p(s)=\operatorname{softmax}(-(q+sd)/T_Q)$. Differentiation gives
$\sum_a|p_a'(s)|=T_Q^{-1}\E_{p(s)}|d-\E_{p(s)}d|\leq\operatorname{osc}(d)/(2T_Q)$.
The last bound follows from Cauchy--Schwarz and the variance bound for a variable in an interval of length $\operatorname{osc}(d)$. Integrate from zero to one.
\end{proof}

Let $\bar p_{\mathcal G}(U)=\E[p_T^{\mathcal G}(\cdot\mid W,S)\mid U]$, where $p_T^{\mathcal G}$ uses the same temperature on the sparse root costs. Let $e_{\mathcal G}(U)=\|\pi_{\mathcal G}(\cdot\mid U)-\bar p_{\mathcal G}(U)\|_\infty$, and define the conditional pruning distortion
\begin{equation}
 s_{\mathcal G}(U)=\E\!\left[
 \operatorname{osc}(\widetilde Q_W^{\mathcal G}(\cdot,S)-\widetilde Q_W(\cdot,S))\mid U\right].
\end{equation}
Random search seeds can be included in this expectation. They are used only in offline target construction.
\begin{corollary}[Sparse-target decision preservation]
\label{cor:sparse_stability}
Under Theorem~\ref{thm:decision_margin}'s assumptions, a sparse-target student selects the same unique online-optimal action at $H$ whenever $\Delta(H)>b(H)$ and
\begin{equation}
 M_T(H)>\frac{s_{\mathcal G}(U)}{2T_Q}+2e_{\mathcal G}(U).
 \label{eq:sparse_margin}
\end{equation}
The regret bound in Equation~\eqref{eq:margin_regret} applies with failure of this augmented condition in place of $\mathcal B_\pi$.
\end{corollary}
\begin{proof}
Lemma~\ref{lem:sparse_softmax} bounds the change of every pairwise projected probability difference by $s_{\mathcal G}(U)/(2T_Q)$. Theorem~\ref{thm:decision_margin}'s proof supplies the full-target margin $M_T(H)$. Student approximation reduces the remaining margin by at most $2e_{\mathcal G}(U)$. The strict inequality preserves the optimum. Equation~\eqref{eq:margin_regret} uses the same actual output rule and online costs and therefore applies unchanged.
\end{proof}
Retaining all first queries at each supervised root is essential to the stated comparison: it keeps the target action space fixed. Deeper graph membership can depend on the complete record, but it is not an input to the deployed controller. The large-$K$ construction and corollary specify an extension. We use the exact lattice teacher in all experiments; the sparse construction is a computational extension. No general pruning-error bound follows from beam width alone.

%% file: sections/appendix_b_benchmarks.tex
\section{Information Benchmarks and Analytical Examples}
\label{app:benchmarks}
\label{app:proofs}
\subsection{Information and prediction conventions}
Let $(\Omega,\mathcal F,P)$ carry one coherent joint distribution of mandatory evidence, optional groups, and labels. The sigma-field $\mathcal F_S$ is generated by the observed history $H_S$. A query refines it to $\mathcal F_{S\cup\{m\}}$. Regular conditional distributions are assumed to exist. All expectations in the online recursion refer to this joint distribution, not to separate incompatible models for different subsets.

The predictor $f(H_S)$ is fixed during the decision analysis. Unless stated otherwise, the proofs below spell out the optional-deferral instantiation: its terminal actions are prediction and deferral, with fixed $\tau>0$. Under compulsory prediction, stopping invokes the predictor and the Bellman argument uses $R(H)=r(H)$ instead. Query costs are finite and nonnegative. A policy is measurable with respect to the current history (and, for randomized policies, independent randomization). No query changes the underlying target or measurements. These assumptions describe retrospective static evidence revelation. Calendar-time waiting, query failures, and unavailable future measurements would require an enlarged state and are not modeled by the present experiment.

% A compressed input can merge histories whose useful actions differ; this contributes to $\mathcal E_{\rm priv}$ in Section~\ref{sec:theory}. 
A compressed input can merge histories with different useful actions, contributing to the irreducible discrepancy $\mathcal E_{\rm priv}$ defined in Appendix~\ref{app:projection}
Observed metadata means metadata genuinely available at decision time, not merely present in the complete record. The ideal terminal rule used below is a decision benchmark; Theorem~\ref{thm:transfer} instead holds the deployed terminal rule $g$ fixed.

\paragraph{Online decision benchmark.}
The learned controller acts without knowing future evidence values. For a fixed predictor $f$, define its conditional prediction risk $r(H)$ and the optimal cost $R(H)$ of the allowed terminal rule by
\begin{equation}
 r(H)=\E[\ell(f(H),Y)\mid H],\quad
 R(H)=\begin{cases}
 r(H),&\text{compulsory prediction},\\
 \min\{r(H),\tau\},&\text{optional deferral}.
 \end{cases}
\label{eq:stop}
\end{equation}
Here $H$ denotes any observed history. With deferral, ties favor deferring. These conditional risks describe an online decision benchmark, whereas the teacher's $\rho_t(S)$ is a realized score. For the objective in Equation~\eqref{eq:objective}, the minimum conditional cost-to-go is
\begin{equation}
 V_\lambda(H_S)=\min\!\left\{R(H_S),\min_{m\in\M\setminus S}
 \left[\lambda c_m+\E[V_\lambda(H_{S\cup\{m\}})\mid H_S]\right]\right\}.
 \label{eq:bellman}
\end{equation}
The expectation averages possible values of the next group, conditional on $H_S$; write $H_m^+=H_{S\cup\{m\}}$ for the enlarged history including that value. At the full subset, $V_\lambda(H_\M)=R(H_\M)$. Backward induction gives the optimum over policies using only observed evidence. Fixing $f$ separates its prediction error from error in the acquisition controller.

\subsection{Terminal loss and acquisition stopping in the Bellman recurrence}
\begin{proof}
At a fixed history, predicting incurs conditional loss $r(H)$ and deferring incurs $\tau$. Their minimum is $R(H)=\min\{r(H),\tau\}$. Prediction is strictly preferred exactly when $r(H)<\tau$; at equality either action is optimal, and the stated tie rule selects deferral.

At the full subset, no query remains, so the optimal value is $R(H_\M)$. Suppose the optimal values at all strict successor subsets are known. Any policy at $H_S$ either stops or first queries an unacquired group $m$. The minimum cost in the first case is $R(H_S)$. In the second case it is $\lambda c_m+\E[V_\lambda(H_{S\cup\{m\}})\mid H_S]$. Minimizing over this exhaustive finite first-action set gives Equation~\eqref{eq:bellman}. Induction in the number of unacquired groups completes the proof. The recurrence includes an expectation precisely because the query's value is not known at the time of choosing it.
\end{proof}

For a fixed imperfect predictor, the one-query value need not be nonnegative: its predictions under different masks can be inconsistent. With a Bayes-optimal predictor and permission to ignore new evidence, information cannot increase minimum expected decision loss before cost. The nonnegativity claim requires those extra conditions; it is not an automatic property of the fitted network.

\subsection{Terminal loss and acquisition value}
\label{app:terminal_value}
For the optional-deferral terminal rule $R(H)=\min\{r(H),\tau\}$, write $H^+_m=H_{S\cup\{m\}}$ and $r^+=r(H^+_m)$. The one-query value before paying acquisition cost is $\Delta_m(H)=R(H)-\E[R(H^+_m)\mid H]$.
\begin{proposition}[Threshold-sensitive acquisition value]
\label{prop:threshold}
For integrable conditional prediction risks,
\begin{equation}
 \Delta_m(H)=r(H)-\E[r^+\mid H]
 +\E[\pos{r^+-\tau}\mid H]-\pos{r(H)-\tau}.
 \label{eq:threshold}
\end{equation}
Equal expected prediction-risk reductions can therefore have different acquisition values.
\end{proposition}
The correction accounts for prediction loss replaced by deferral.

\begin{proof}
Use $\min\{u,\tau\}=u-[u-\tau]_+$ at both the current and successor risks. Then
\begin{align*}
R(H)-\E[R(H^+_m)\mid H]
&=r(H)-[r(H)-\tau]_+
 -\E[r^+-[r^+-\tau]_+\mid H]\\
&=r(H)-\E[r^+\mid H]+
 \E[[r^+-\tau]_+\mid H]-[r(H)-\tau]_+.
\end{align*}
Integrability justifies the rearrangement. The two observation channels constructed below have equal expected risk reductions and different values of the correction term, proving the final statement.
\end{proof}

\subsection{Observation channels realizing the numerical examples}
\label{app:examples}
\paragraph{Equal average risk, unequal selective value.}
Let $Y\sim\operatorname{Bernoulli}(1/2)$. For query A, reveal a binary symmetric channel output with independent flip probability $0.2$. The posterior Bayes error is $0.2$ for either output. For query B, draw an independent $Z\sim\operatorname{Bernoulli}(1/2)$ and reveal $Z$ as part of the observation. If $Z=1$, reveal $Y$ exactly; if $Z=0$, reveal a binary symmetric channel with flip probability $0.4$. Thus posterior error is $0$ with probability $1/2$ and $0.4$ otherwise. Both queries have expected prediction error $0.2$.

With $\tau=1/4$, the current selective risk is $1/4$. Query A's expected stopping risk is $1/5$, giving value $1/20=0.05$. Query B's expected stopping risk is $\tfrac12(0)+\tfrac12(1/4)=1/8$, giving value $1/8=0.125$. The current prediction error is $0.5$. At acquisition price $\lambda c=2/25=0.08$, only query B is worth buying, despite the equal expected prediction errors. Each channel is coherent with the same balanced prior; these are not arbitrary posterior-risk assignments.

\paragraph{Complementary evidence.}
Let $X^{(1)},X^{(2)}$ be independent fair bits and $Y=X^{(1)}\oplus X^{(2)}$. A Bayes predictor has error $1/2$ before any query and after either single query, and zero after both. With $\tau=0.3$ and cost $0.05$ per group, every one-step gain is zero, so a myopic stop/query rule stops. The two-query plan has total value $0.1$, less than stopping at $0.3$. With compulsory prediction, stopping costs $0.5$ and a single query followed by immediate prediction costs $0.05+0.5=0.55$, whereas the two-query continuation costs $0.1$. The preference for multi-step continuation therefore does not rely on deferral. A subset-lattice planner can represent either continuation.

\paragraph{Teacher-majority counterexample.}
Take a balanced binary label, an independent success event $Z$ of probability $0.6$, and one observation returning $(Z,Y)$ on success and a failure symbol otherwise. On success the terminal risk is zero; on failure it is $0.3$, because deferral beats error $0.5$. The online query cost is $0.2+0.4(0.3)=0.32$, so the optimum is to defer at $0.3$. The complete-world teacher pays $0.2$ on success and otherwise stops at $0.3$, with relaxed mean value $0.6(0.2)+0.4(0.3)=0.24$. The information gap is $0.06$.

Under hard-label cross-entropy distillation from a constant initial history, the population-optimal student's query probability is $0.6$. Argmax decoding therefore queries and incurs regret $0.02$. Sampling from the predicted probabilities also does not recover the optimal defer-only rule: its expected objective is $0.6(0.32)+0.4(0.3)=0.312$. This is a hard-label inconsistency example. The same observation model also permits a soft-target example. For temperature $T_Q=0.02$, the teacher's query probability is $\sigma(0.1/T_Q)$ on success and $\sigma(-0.2/T_Q)$ on failure, where $\sigma(v)=(1+e^{-v})^{-1}$. The exact projected probability is $0.6\sigma(5)+0.4\sigma(-10)\approx0.5960>1/2$, so its argmax still queries. Here $\mathcal E_{\rm fit}=0$ but $\mathcal E_{\rm priv}>0$; the transfer bound does not imply zero regret. The conclusion is specific to the stated observation model and temperature.

\subsection{Complete-record information advantage}
\label{app:info}
To isolate anticipation from score-estimation error, first use the true stopping cost $R(H_S)$ inside the realized teacher. For complete world $W$, call its value $V^{\rm cl}(W,S)$ and define
\begin{equation}
 L(H_S)=\E[V^{\rm cl}(W,S)\mid H_S],\quad
 j(H)=\min_a\E[Q_a^{\rm cl}\mid H]-\E[\min_a Q_a^{\rm cl}\mid H],
 \label{eq:gap}
\end{equation}
where $Q_{\stopa}^{\rm cl}=R(H)$ and $Q_m^{\rm cl}=\lambda c_m+V^{\rm cl}(W,S\cup\{m\})$.
\begin{theorem}[Local information-relaxation gap]
\label{thm:gap}
For finite $K$, nonnegative finite costs, and integrable terminal risks, $L(H)\leq V_\lambda(H)$ and $j(H)\geq0$. With $d(H)=V_\lambda(H)-L(H)$,
\begin{equation}
 j(H)\leq d(H)\leq j(H)+\max\!\left\{0,\max_{m\notin S}\E[d(H^+_m)\mid H]\right\}.
 \label{eq:gapbound}
\end{equation}
At the full subset, $d=0$. If $\bar j_i$ bounds $j$ at histories with $i$ unacquired groups, then $d(H)\leq\sum_{i=1}^{K-|S|}\bar j_i$.
\end{theorem}
The gap is the advantage of choosing after a future realization is known. Realized outcome-dependent terminal scores add a separate discrepancy; they are not identified with conditional risks by this theorem.

% \paragraph{Proof.}
\begin{proof}
A complete-world planner minimizes the realized path cost among all subset paths, while every nonanticipative policy chooses one of those paths. Its complete-world minimum is therefore no larger than the realized conditional-risk path cost of any online policy. Taking conditional expectations and then minimizing over online policies gives $L(H)\leq V_\lambda(H)$. The same fact follows by backward induction using the tower property.

At history $H$, put $u_a=\E[Q_a^{\rm cl}\mid H]$. Since minimization is concave as a function of the action-cost vector,
\[
 j(H)=\min_a u_a-\E[\min_a Q_a^{\rm cl}\mid H]\geq0.
\]
Let $b_{\stopa}=0$ and $b_m=\E[d(H^+_m)\mid H]\geq0$. The tower property yields $\E[V^{\rm cl}(W,S\cup\{m\})\mid H]=\E[L(H^+_m)\mid H]$, hence
\[
 V_\lambda(H)=\min_a(u_a+b_a),\qquad
 d(H)=j(H)+\min_a(u_a+b_a)-\min_a u_a.
\]
The last difference lies between zero and $\max_a b_a$. This is Equation~\eqref{eq:gapbound}. At the full subset, both planners must stop, so $d=0$. Applying the upper bound recursively with depth-wise bounds $\bar j_i$ gives $d(H)\leq\sum_{i=1}^{K-|S|}\bar j_i$.
\end{proof}

\subsection{Terminal-score perturbations}
\label{app:surrogate}
Theorem~\ref{thm:gap} uses $R(H)$ in the complete-world teacher to isolate the information advantage. The implementation instead uses realized terminal supervision $\rho_t(S)$. Suppose, as a deterministic additional assumption, that every subset in a complete episode obeys $|\rho_t(S)-R(H_{t,S})|\leq\eta$. Then
\begin{equation}
 |\widetilde V_{\lambda,t}(S)-V^{\rm cl}(W_t,S)|\leq\eta
 \label{eq:scorebound}
\end{equation}
for every subset. To prove this, use the boundary inequality at the full set and the sup-norm nonexpansiveness of a minimum at each earlier subset. Each path pays only one terminal loss, so this particular perturbation does not accumulate once per query.

There is also a population-level order-of-operations issue. Even if an individual loss $L_t$ is an unbiased sample of conditional prediction risk, concavity gives
\begin{equation}
 \E[\min\{\tau,L_t\}\mid H]\leq\min\{\tau,\E[L_t\mid H]\}.
 \label{eq:clipbias}
\end{equation}
For example, if $L_t$ is zero or one with equal probability and $\tau=0.3$, the left side is $0.15$ while the right side is $0.3$. Predicting whether an individual realized loss falls below a threshold is not the same as estimating conditional expected loss. Thus unbiased prediction losses do not justify setting the terminal mismatch $\delta$ in Equation~\eqref{eq:terminal_mismatch} to zero after clipping or a different terminal-output rule.

%% file: sections/appendix_c_results.tex
\section{Additional Experimental Results}
\label{app:extended}
The following tables report the complete acquisition-price profiles and additional predictive measurements for the same implementation used in the main experiments.

\subsection{Acquisition-price profiles}
\label{app:prices}
Table~\ref{tab:prices} gives all price profiles evaluated in Section~\ref{sec:exp_prices}. The predictor and terminal procedure remain fixed; each profile has its own development-set multiplier calibration and student redistillation.

\begin{table}[!htbp]
\centering
\caption{\textbf{Policy replanning with a fixed predictor.} Each price profile has its own multiplier calibration and student redistillation. ``Changed'' is the reported action-change rate relative to the default policy. Each row's mean expenditure uses its own prices, not a common-price comparison.}
\label{tab:prices}
\small
\begin{tabular}{lrrr}
\toprule
Price profile & Mean cost & Phase F1 $\uparrow$ & Changed\\
\midrule
\textbf{Default} & \textbf{2.179} & \textbf{0.983} & \textbf{0.0\%}\\
ED $0.8\times$ & 2.280 & 0.987 & 37.0\%\\
ED $0.9\times$ & 2.103 & 0.984 & 25.7\%\\
ED $1.1\times$ & 2.192 & 0.983 & 51.0\%\\
ED $1.2\times$ & 2.450 & 0.987 & 48.3\%\\
ED $1.5\times$ & 2.715 & 0.984 & 44.9\%\\
ED $2.0\times$ & 3.166 & 0.981 & 51.2\%\\
HOSP $0.5\times$ & 2.053 & 0.983 & 47.6\%\\
SRC $3.0\times$ & 2.150 & 0.982 & 63.6\%\\
Full cost-rank swap & 1.768 & 0.985 & 36.2\%\\
\bottomrule
\end{tabular}
\end{table}

\subsection{Additional predictive measurements}
\begin{table}[!htbp]
\centering
\caption{Alert-prediction metrics under the default acquisition costs.}
\label{tab:alerts}
\small
\begin{tabular}{lrrr}
\toprule
System & Mean cost & Alert AUROC & Alert Brier\\
\midrule
XGBoost (WW) & 0.000 & 0.933 & 0.094\\
SelectiveNet (WW) & 0.000 & 0.906 & 0.117\\
MLP (WW+SRC) & 0.400 & 0.907 & 0.102\\
ReAct Router & 0.472 & 0.919 & 0.121\\
MLP (WW+ED) & 1.000 & 0.975 & 0.059\\
\midrule
\textbf{\SLD\ (ours)} & \textbf{2.179} & \textbf{0.966} & \textbf{0.068}\\
\bottomrule
\end{tabular}
\end{table}
% The aggregate evaluation also reports lower alert Brier scores for GDFS and ACO than for \SLD, without numerical values for those comparisons. 
Table~\ref{tab:alerts} reports alert AUROC and Brier scores. MLP (WW+ED) achieves higher alert AUROC and lower Brier score than SLD, at lower acquisition cost.
Rankings can differ across phase prediction, alert prediction, and calibration; these endpoints are therefore reported separately.

\begin{table}[!htbp]
\centering
\caption{Evidence-representation comparison at default prices. The full model uses E5-Mistral. The text representation supplies evidence features, not the acquisition policy.}
\label{tab:encoders}
\small
\begin{tabular}{lrr}
\toprule
Frozen representation & Mean cost & Phase F1\\
\midrule
\textbf{\SLD\ (ours)} & \textbf{2.179} & \textbf{0.983}\\
gte-Qwen2 replacement & 2.134 & 0.972\\
W/O frozen text adapter & 2.263 & 0.963\\
\bottomrule
\end{tabular}
\end{table}

% Teacher-sequence alignment is reported as of $0.599$ for \SLD, $0.469$ for EDDI, $0.365$ for GDFS, $0.355$ for ACO, $0.323$ for Supervised Router, and $0.056$ for ReAct Router. 
SLD achieves a teacher-sequence alignment score of 0.599, compared with 0.469 for EDDI, 0.365 for GDFS, 0.355 for ACO, 0.323 for Supervised Router, and 0.056 for ReAct Router.
This is agreement with realized teacher behavior. 

% As the counterexample in Appendix~\ref{app:examples} shows, imitation fidelity cannot be substituted for online decision optimality. The precise alignment functional is required to reproduce this secondary diagnostic.

%% file: sections/appendix_d_protocol.tex
\section{Implementation Details}
\label{app:implementation}
The implementation follows the shared subset evaluator, cost-sensitive lattice teacher, and observed-history student in Sections~\ref{sec:subset}--\ref{sec:distillation}.

\subsection{Evidence preprocessing and frozen text adapter}
\label{app:evidence_implementation}
The fixed block order is WW, ED, HOSP, POS, and SRC. Structured features use training-split standardization, categorical fields use one-hot encoding, and week-of-year uses sine/cosine features. Development and test records use the same preprocessing. An observation mask identifies acquired evidence groups.

\paragraph{Frozen text adapter.}
Short textual or categorical summaries are represented by the frozen E5-Mistral-7B-Instruct encoder (\texttt{intfloat/e5-mistral-7b-instruct}). The task instruction is projected to 16 dimensions and each observed evidence summary to 64 dimensions, denoted \emph{i16/e64}. Unqueried evidence embeddings are zero-filled. Cached, prompt-conditioned encoder outputs are reused across runs; its frozen parameters are excluded from trainable counts. The observed-history controller, rather than the text encoder, chooses queries.

\subsection{Shared subset evaluation}
\label{app:subset_implementation}
For the shared masked predictor in Section~\ref{sec:subset}, $r_{t,S}\in\mathbb R^{d_r}$ denotes the input assembled from history $H_{t,S}$. The Gaussian encoder uses a residual MLP with width 128, two residual blocks, layer normalization, and dropout 0.10. It outputs the mean and log-variance used in Equation~\eqref{eq:encoder}:
\begin{equation}
 \operatorname{enc}_{\theta}(r_{t,S})=(\mu_{t,S},\log\sigma^2_{t,S}),\qquad
 \mu_{t,S},\log\sigma^2_{t,S}\in\mathbb R^{16}.
 \label{eq:implementation_encoder}
\end{equation}
The log-variance output is unconstrained. The decoder $\operatorname{decode}_\psi$ takes $[\mu_{t,S};\log\sigma^2_{t,S}]\in\mathbb R^{32}$ and uses deterministic natural-parameter-network-style moment heads \citep{wang2016npn} for four-class phase and binary alert outputs. Table~\ref{tab:sld_implementation_card} summarizes the learned components.

\begin{table}[!htbp]
\centering
\caption{\textbf{\SLD{} architecture.} Dimensions refer to i16/e64.}
\label{tab:sld_implementation_card}
\small
\setlength{\tabcolsep}{4pt}
\begin{tabularx}{\linewidth}{@{}>{\raggedright\arraybackslash}p{0.17\linewidth}>{\raggedright\arraybackslash}p{0.23\linewidth}>{\raggedright\arraybackslash}p{0.31\linewidth}>{\raggedright\arraybackslash}X@{}}
\toprule
Component & Input & Architecture & Output\\
\midrule
Frozen text adapter & Task instruction; observed summaries & Frozen E5-Mistral-7B-Instruct; 16-/64-dimensional projections & Task and evidence representations\\
Masked Gaussian encoder & Masked vector $r_{t,S}\in\mathbb R^{d_r}$ & Width 128; two residual blocks; layer norm; dropout 0.10 & 16 means and 16 log-variances\\
Moment decoder & Gaussian moments $[\mu_S;\log\sigma_S^2]\in\mathbb R^{32}$ & Deterministic NPN-style moment heads & Phase distribution; alert probability\\
Student controller & 60-dimensional input $s(H_S)$ & Width 128; one residual block; dropout 0.10 & Five query/stop logits; terminal answer probability\\
\bottomrule
\end{tabularx}
\end{table}

\paragraph{Predictor optimization.}
The encoder and decoder are trained for 1,000 gradient steps with batch size 64, Adam learning rate $10^{-3}$, and KL coefficient $\beta=0.02$ in Equation~\eqref{eq:state}. The probabilities of sampling the mandatory-evidence-only mask, a proper nonempty partial mask, and the full-evidence mask are
\begin{equation}
 (p_{\varnothing},p_{\mathrm{partial}},p_{\mathrm{full}})=(0.10,0.30,0.60).
 \label{eq:predictor_mask_mixture}
\end{equation}
This mixture retains wastewater-only examples while emphasizing full-evidence prediction. The checkpoint is selected on development data. Observed-history distillation uses the uniform mask distribution specified in Section~\ref{sec:distillation}.

\subsection{Cost-sensitive lattice teacher}
\label{app:teacher_implementation}
The selected predictor parameters $(\theta,\psi)$ are frozen before scoring all $2^4=16$ subsets of each labeled training and development record. The task-specific terminal procedure supplies the scores $\rho_t(S)$ and answer indicators $d_t(S)$ used in Section~\ref{sec:subset}. Terminal scores apply training-split log-risk calibration to negative log-likelihood plus subset KL, with deferral cost $\tau=3.0$. Each target budget has a separate development-selected acquisition-cost multiplier.

\paragraph{Multiplier search.}
For budget $B$, the multiplier $\lambda_B$ is the value of $\lambda$ selected on the development split from
\begin{equation}
\begin{split}
 \Lambda=\{&0.02,0.05,0.075,0.1,0.15,0.2,0.25,0.35,0.5,0.75,\\
           &1,1.5,2,3,5,8,13,21,34,55,89\}.
\end{split}
\label{eq:implementation_lambda_grid}
\end{equation}
The benchmark budgets are $B_0=0.0$, $B_1=1.2$, $B_2=2.5$, and $B_3=4.1$, under the default costs $c_{\mathrm{ED}}=1.0$, $c_{\mathrm{HOSP}}=1.5$, $c_{\mathrm{POS}}=1.2$, and $c_{\mathrm{SRC}}=0.4$. The lattice teacher applies Equation~\eqref{eq:teacher} to each complete record; hard stop/query actions and soft preferences $p_T$ from Equation~\eqref{eq:soft} are cached for distillation. Deployment uses the learned policy and terminal heads.

\subsection{Observed-history distillation and deployment}
\label{app:distillation_implementation}
The student input $s(H_{t,S})$ combines encoder moments, the optional-group mask, the task embedding, budget-cost features, and a reliability summary. For i16/e64 its dimension is
\begin{equation}
 d_{\pi}=\underbrace{2\times16}_{\text{moments}}+
 \underbrace{4}_{\text{mask}}+
 \underbrace{16}_{\text{task}}+
 \underbrace{4}_{\text{budget/cost}}+
 \underbrace{4}_{\text{reliability}}=60.
 \label{eq:implementation_policy_dimension}
\end{equation}
The residual policy MLP has width 128, one residual block, and dropout 0.10. Its heads produce five query/stop logits for $\pi_\eta$ and a terminal answer probability $h_\omega$, as in Section~\ref{sec:distillation}. Repeated or unavailable queries are masked, stop is always valid, and the full subset terminates.

\paragraph{Student optimization.}
The policy and terminal head use 1,000 gradient steps, batch size 64, and Adam learning rate $10^{-3}$. They minimize Equation~\eqref{eq:distill}, combining cross-entropy against soft preferences $p_T$ with terminal answer-head BCE. Equation~\eqref{eq:soft} smooths near-tied teacher costs. Each update draws a fresh mask from $w=\operatorname{Unif}(2^\M)$ independently of the record and multiplier and uses the corresponding cached target. Validation and test metrics retain the original episode distribution. Table~\ref{tab:sld_training_card} collects the training settings.

\begin{table}[!htbp]
\centering
\caption{\textbf{\SLD{} training settings.} Network settings describe the masked encoder and student policy, respectively.}
\label{tab:sld_training_card}
\small
\setlength{\tabcolsep}{4pt}
\begin{tabularx}{\linewidth}{@{}>{\raggedright\arraybackslash}p{0.33\linewidth}>{\raggedright\arraybackslash}X>{\raggedright\arraybackslash}X@{}}
\toprule
Setting & Shared subset evaluation & Observed-history distillation\\
\midrule
Hidden width & 128 & 128\\
Residual blocks & 2 & 1\\
Dropout probability & 0.10 & 0.10\\
Gradient steps & 1,000 & 1,000\\
Batch size & 64 & 64\\
Optimizer & Adam & Adam\\
Learning rate & $10^{-3}$ & $10^{-3}$\\
KL coefficient & $\beta=0.02$ & ---\\
Mask selection & Empty/partial/full: $0.10/0.30/0.60$ & Repeated uniform draws from $2^\M$\\
Model selection & Development checkpoint & Development calibration\\
\bottomrule
\end{tabularx}
\end{table}

\paragraph{Deployment calibration.}
The development search for the action-logit correction uses the grid
\begin{equation}
 \mathcal C=\{0,0.025,0.05,0.075,0.1,0.125,0.15,0.2\}.
 \label{eq:implementation_logit_grid}
\end{equation}
The selected adjustment is fixed before evaluation; the calibrated policy uses only acquired evidence.

\subsection{Hardware and execution}
\label{app:hardware}
Trainable neural components were optimized on NVIDIA RTX PRO 6000 Blackwell PCIe GPUs and NVIDIA A100 80\,GB PCIe GPUs. Outputs of the frozen E5-Mistral text encoder are cached and reused. Neither this text encoder nor the local ReAct text model is fine-tuned; the masked Gaussian encoder, moment decoder, and student controller are trainable.

\subsection{Temporal protocol and reported expenditure}
\label{app:budget}
The benchmark uses fixed state--week cutoffs and chronological partitions. Queries expose only evidence available by the episode cutoff; record construction is fixed before fitting. Development data select operating points, and evaluation reports realized student expenditure.

The multiplier $\lambda$ penalizes evidence cost; it is not a hard budget. A per-episode cap would additionally require remaining budget in the state and removal of unaffordable queries. For exact optimizers $\pi_1,\pi_2$ at $0\leq\lambda_1<\lambda_2$, let $L_i$ and $C_i$ denote their expected terminal loss and evidence cost. Optimality gives
\[
 L_1+\lambda_1 C_1\leq L_2+\lambda_1 C_2,\qquad
 L_2+\lambda_2 C_2\leq L_1+\lambda_2 C_1.
\]
Adding yields $(\lambda_2-\lambda_1)(C_2-C_1)\leq0$, hence $C_2\leq C_1$. The argument fixes relative prices and varies only $\lambda$; expenditures under distinct price vectors are reported directly in Table~\ref{tab:prices}.